\documentclass[journal]{IEEEtran}

\usepackage{amsmath, amssymb, amsthm}
\usepackage{graphicx}
\usepackage{booktabs}
\usepackage[
colorlinks=true,
linkcolor=blue,
citecolor=blue,
urlcolor=cyan,
]{hyperref}
\usepackage[T1]{fontenc}
\usepackage{times}
\usepackage{comment}
\usepackage{subcaption}
\usepackage{float}
\usepackage{algorithm}
\usepackage{algpseudocode}
\usepackage{stfloats} 

\newcommand{\R}{\mathbb{R}}
\newcommand{\sspace}{\mathcal{S}}
\newcommand{\aspace}{\mathcal{A}}
\newcommand{\flowmap}{\Phi}
\newcommand{\CauchyGreen}{C}
\newcommand{\Jacobian}{J}
\newcommand{\FTLE}{\sigma}

\newcommand{\defeq}{\triangleq}
\newcommand{\norm}[1]{\left\lVert#1\right\rVert}
\newcommand{\abs}[1]{\left\lvert#1\right\rvert}
\newcommand{\tr}{^T}
\newcommand{\mbr}{\mathrm{MBR}}
\newcommand{\asas}{\mathrm{ASAS}}
\newcommand{\tasas}{\mathrm{TASAS}}
\newcommand{\ms}{\mathrm{MS}}

\newtheorem{definition}{Definition}[section]
\newtheorem{proposition}{Proposition}[section]

\title{\textbf{A Dynamical Systems Framework for Reinforcement Learning Safety and Robustness Verification}}

\author{Ahmed Nasir,  Abdelhafid Zenati\thanks{A. Nasir and A. Zenati are with Engineering Department, School of Science and Technology (SST), City University of London, Northampton Square,  
		London, EC1V 0HB, United Kingdom.}%
}

\begin{document}

	\maketitle
	
	\begin{abstract}
		The application of reinforcement learning to safety-critical systems is limited by the lack of formal methods for verifying the robustness and safety of learned policies. This paper introduces a novel framework that addresses this gap by analyzing the combination of an RL agent and its environment as a discrete-time autonomous dynamical system. By leveraging tools from dynamical systems theory, specifically the Finite-Time Lyapunov Exponent (FTLE), we identify and visualize Lagrangian Coherent Structures (LCS) that act as the hidden "skeleton" governing the system's behavior. We demonstrate that repelling LCS function as safety barriers around unsafe regions, while attracting LCS reveal the system's convergence properties and potential failure modes, such as unintended "trap" states. To move beyond qualitative visualization, we introduce a suite of quantitative metrics, Mean Boundary Repulsion (MBR), Aggregated Spurious Attractor Strength (ASAS), and Temporally-Aware Spurious Attractor Strength (TASAS), to formally measure a policy's safety margin and robustness. We further provide a method for deriving local stability guarantees and extend the analysis to handle model uncertainty. Through experiments in both discrete and continuous control environments, we show that this framework provides a comprehensive and interpretable assessment of policy behavior, successfully identifying critical flaws in policies that appear successful based on reward alone.
	\end{abstract}
	

	\section{Introduction}
	
	Artificial Intelligence (AI) aims to create machines capable of reasoning, learning, and acting autonomously, mirroring the faculties of human intelligence \cite{norvig1994artificial}. A driving force behind modern AI is Machine Learning (ML), a field dedicated to developing algorithms that enable systems to learn from data and experience rather than being explicitly programmed for a task \cite{shinde2018review}. Within ML, the advent of Deep Learning, which utilizes deep neural networks with many layers, has triggered a revolution, achieving unprecedented performance in perception-heavy domains like computer vision and natural language processing \cite{lecun2015deep, alzubaidi2021review}. Parallel to this, Reinforcement Learning (RL) has emerged as a powerful framework for tackling sequential decision-making problems, where an agent learns to achieve a goal in an uncertain environment through a process of trial and error \cite{sutton1998reinforcement, kaelbling1996reinforcement}.
	
	The fusion of these fields has given rise to Deep Reinforcement Learning (DRL), a paradigm that leverages the powerful representation learning of deep networks to solve complex RL problems with high-dimensional state spaces \cite{arulkumaran2017deep, prudencio2023survey}. DRL has produced remarkable successes, achieving superhuman performance in complex strategic games like Go, Chess, and StarCraft \cite{silver2016mastering, silver2018general, vinyals2019grandmaster}, and demonstrating significant potential in fields as diverse as robotics \cite{kober2013reinforcement, openai2019solving}, autonomous driving \cite{kiran2021deep, shalev2016safe}, and fluid dynamics \cite{krishna2023finite, zolman2024sindyrl}. Foundational DRL algorithms, including Deep Q-Networks (DQN) \cite{mnih2013playing}, Proximal Policy Optimization (PPO) \cite{schulman2017proximal}, and Soft Actor-Critic (SAC) \cite{haarnoja2018soft}, have become standard tools for training agents that can learn sophisticated policies directly from raw sensory inputs.
	
	However, despite these achievements in simulated or controlled environments, the transition of DRL agents to safety-critical, real-world systems remains a formidable challenge \cite{doshi2017towards, pullum2020review, prudencio2023survey, kiran2021deep}. Unlike the forgiving nature of a video game, applications in autonomous navigation or medical robotics demand stringent guarantees of safety and robustness. Policies learned by DRL agents often function as "black boxes," making their decision-making processes opaque and difficult to trust \cite{attar2019reinforcement, zolman2024sindyrl}. This lack of interpretability is compounded by a critical vulnerability: the susceptibility of deep neural networks to adversarial perturbations. It has been repeatedly demonstrated that minuscule, often imperceptible, changes to an agent's sensory input can cause catastrophic failures in its policy, leading to unsafe and unpredictable behavior \cite{zhang2020robust, schott2024robust, fischer2019online, zhang2021robust, dabholkar2023adversarial, fischer2019online}. This fragility poses a significant barrier to deployment, as real-world sensors are inherently noisy and environments can be unpredictable or even adversarial \cite{zhang2021robust, young2024enhancing}.
	
	The DRL community has actively pursued solutions to this robustness problem. Strategies range from safe RL, which often uses constrained optimization or Lyapunov-based methods to ensure stability \cite{chow2018lyapunov, achiam2017constrained}, to robust RL, which formulates the problem as a zero-sum game against an adversary to find policies that are resilient to worst-case disturbances \cite{pinto2017robust, li2023safe}. Formal verification techniques have also been explored, using methods like abstract interpretation to provide probabilistic guarantees on policy safety \cite{bacci2022verified, wu2024verified}. While these approaches have made significant strides, they often rely on statistical assumptions, are computationally intensive, or do not provide a clear, intuitive picture of the policy's global behavior and potential failure modes \cite{prudencio2023survey, schott2024robust, wu2024verified}. A critical gap remains for frameworks that can offer deterministic, interpretable, and formal guarantees of a policy's safety and robustness.
	
	This paper addresses this gap by introducing a novel verification framework rooted in the principles of dynamical systems theory. Our central premise is that the combination of a deterministic RL policy and its environment constitutes a discrete-time autonomous dynamical system. The behavior of any agent within this system, how it navigates, where it succeeds, and where it might fail, can be understood by studying the geometry of the flow it induces on the state space. This perspective allows us to move beyond simple performance metrics, like cumulative reward, and probe the underlying structures that govern trajectory evolution.
	
	To achieve this, we employ a key tool from the analysis of complex flows: the \textbf{Finite-Time Lyapunov Exponent (FTLE)}. The FTLE is a scalar field that measures the maximal rate at which initially close trajectories separate over a finite time window. Regions with high FTLE values form sharp ridges known as \textbf{Lagrangian Coherent Structures (LCS)}, which function as the hidden "skeleton" of the dynamics \cite{haller2015lagrangian}. These structures act as the most influential material surfaces, partitioning the flow into distinct, coherent regions and fundamentally governing the transport and mixing properties of the system \cite{shadden2005definition, jones2024mode}.
	
	\begin{figure}[!t]
		\centering
		\includegraphics[width=\columnwidth]{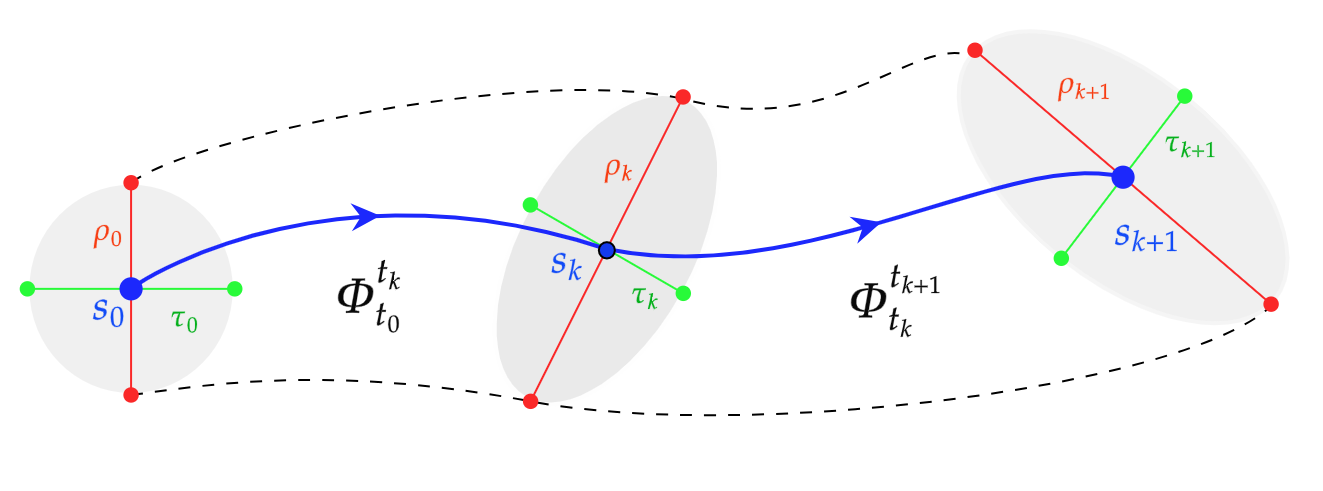}
		\caption{Illustration of trajectory separation, the fundamental principle quantified by the Finite-Time Lyapunov Exponent (FTLE). Initially nearby particles (left) follow distinct paths over time. The maximal rate of their separation (right) reveals the underlying stability and structure of the dynamical flow, forming the basis for identifying Lagrangian Coherent Structures.}
		\label{fig:particle_movement}
	\end{figure}
	
	The power of this framework lies in the direct and intuitive mapping of these dynamical structures to critical properties of an RL agent's policy:
	\begin{itemize}
		\item \textbf{Safety Verification via Repelling LCS:} By computing the FTLE field for the forward-time flow, we identify repelling LCS. These structures function as dynamical barriers or "divides" that nearby trajectories cannot cross. For a safe policy, we hypothesize that the agent must learn to create strong repelling LCS that effectively "wall off" obstacles and unsafe regions of the state space. Visualizing these barriers provides a powerful qualitative assessment of the policy's safety mechanisms.
		\item \textbf{Robustness Verification via Attracting LCS:} Conversely, attracting LCS act as dynamical "highways" and attractors. For a robust policy, we expect to find a single, dominant basin of attraction centered on the desired goal state. The existence of other strong attractors reveals potential robustness failures, such as unintended "trap" states or suboptimal cyclic behaviors where the agent could become permanently stuck.
	\end{itemize}
	
	Building on this qualitative foundation, this paper introduces a suite of quantitative tools to formalize the verification process. Our primary contributions are:
	\begin{enumerate}
		\item A formal framework for analyzing a deterministic RL policy and its environment as a discrete-time dynamical system.
		\item The application of FTLE and LCS to qualitatively visualize the safety barriers and state-space attractors that define a learned policy's behavior.
		\item The development of novel scalar metrics, Mean Boundary Repulsion (MBR), Aggregated Spurious Attractor Strength (ASAS), and Temporally-Aware Spurious Attractor Strength (TASAS), to quantify the policy's safety margin and its robustness of convergence to the goal.
		\item A formal, local stability guarantee derived from the FTLE field, which provides an $(\epsilon, \delta)$-style certificate of robustness against small state perturbations.
	\end{enumerate}
	
	This paper is structured to guide the reader from foundational concepts to advanced applications. We begin by formally establishing the dynamical systems preliminaries, then introduce our quantitative metrics, present our formal stability results, and conclude with advanced robustness analysis and experimental validation across multiple environments.
	
	\section{Reinforcement Learning through the Lens of Dynamical Systems}
	
	We begin by conceptualizing the interaction between a trained reinforcement learning (RL) agent and its environment as a dynamical system. This perspective allows us to analyze the global behavior and stability of the learned policy. Let the agent's state space be a subset of an n-dimensional space, $\sspace \subseteq \R^n$, and let its action space be $\aspace$. A trained, deterministic RL policy can be modeled as a map $\pi: \sspace \to \aspace$, which prescribes an action for every possible state. The environment's dynamics are described by a transition function $T: \sspace \times \aspace \to \sspace$, which dictates the next state given the current state and action.
	
	\subsection{The Closed-Loop System and the Flow Map}
	When the policy $\pi$ is deployed, it creates a closed-loop system with the environment. In this system, the state at the next time step, $s_{k+1}$, is determined solely by the current state $s_k$. This defines an autonomous, discrete-time dynamical system governed by the function $f$:
	\begin{equation}
		s_{k+1} = f(s_k) \defeq T(s_k, \pi(s_k)).
	\end{equation}
	To understand the system's behavior over longer periods, we introduce the \textbf{flow map}, $\flowmap$. This map describes the evolution of the system over a fixed time horizon of $T_{\text{int}}$ steps, advancing an initial state $s_0$ to its final state:
	\begin{equation}
		\flowmap(s_0) = f^{T_{\text{int}}}(s_0) \defeq \underbrace{f \circ f \circ \dots \circ f}_{T_{\text{int}} \text{ times}}(s_0).
	\end{equation}
	
	\subsection{Quantifying Trajectory Separation with the FTLE}
	A key question in dynamical systems is how nearby trajectories separate over time. This separation reveals the stability and structure of the flow. The \textbf{Finite-Time Lyapunov Exponent (FTLE)} is a scalar field that quantifies the maximum rate of separation of infinitesimally close trajectories over a finite interval.
	
	To compute the FTLE, we must first characterize how an infinitesimal neighborhood around a point $s_0$ is deformed by the flow. This local deformation is captured by the \textbf{deformation gradient tensor}, which is the Jacobian of the flow map:
	\begin{equation}
		\Jacobian(s_0) = \frac{\partial \flowmap(s_0)}{\partial s_0}.
	\end{equation}
	The Jacobian $\Jacobian(s_0)$ describes both the stretching and rotation of the neighborhood. To isolate the stretching, which governs trajectory separation, we use the \textbf{right Cauchy-Green deformation tensor}, $\CauchyGreen$:
	\begin{equation}
		\CauchyGreen(s_0) = \Jacobian(s_0)\tr \Jacobian(s_0).
	\end{equation}
	This symmetric, positive-semidefinite matrix removes the rotational component. Its largest eigenvalue, $\lambda_{\max}(\CauchyGreen(s_0)$, represents the maximum squared stretching initiated at $s_0$.
	
	\begin{definition}[Finite-Time Lyapunov Exponent]
		The FTLE field, $\FTLE(s_0, T_{\text{int}})$, is defined as the average exponential rate of this maximal stretching:
		\begin{equation}
			\FTLE(s_0, T_{\text{int}}) = \frac{1}{\abs{T_{\text{int}}}} \ln \sqrt{\lambda_{\max}(\CauchyGreen(s_0))}.
			\label{eq:ftle}
		\end{equation}
	\end{definition}
	The FTLE field reveals the hidden structure of the dynamics. Ridges of high FTLE values, which are locally maximizing curves, are identified as \textbf{Lagrangian Coherent Structures (LCS)}. These structures act as the organizing skeleton of the flow, forming barriers that partition the state space into dynamically distinct regions.
	
	\subsection{Numerical FTLE Calculation in a Discrete State Space}
	In a practical RL setting with a grid-based state space, the flow map $\flowmap$ is not an analytical function, so its Jacobian cannot be computed directly. We therefore approximate the Jacobian using a \textbf{finite difference scheme}.
	
	For each point $s_0$ on the grid, we consider its immediate neighbors. Specifically, for a small step size $h$ (typically the grid spacing), we define perturbed initial points along each coordinate axis $i$:
	$s_{0,i} = s_0 + h \mathbf{e}_i$, where $\mathbf{e}_i$ is the unit vector in the $i$-th dimension.
	
	We advect both the central point $s_0$ and each perturbed point $s_{0,i}$ forward in time by $T_{\text{int}}$ using the learned flow map $\flowmap$. The final perturbation vectors, which describe how the initial orthogonal perturbations have been stretched and rotated by the flow, are:
	\begin{equation}
		\mathbf{u}_i = \flowmap(s_{0,i}) - \flowmap(s_0).
	\end{equation}
	These vectors allow us to construct an approximation of the deformation gradient tensor $\Jacobian$. Each column of the approximate Jacobian is the transformed perturbation vector, normalized by the initial step size: $\Jacobian \approx \frac{1}{h}[\mathbf{u}_1 | \mathbf{u}_2 | \dots | \mathbf{u}_n]$. From this matrix, we compute the Cauchy-Green tensor $\CauchyGreen$ and the FTLE value for $s_0$ using Eq. \eqref{eq:ftle}.
	
	\subsection{Interpreting LCS for Policy Analysis}
	The LCS framework provides a powerful lens for analyzing the safety and robustness of an RL policy by revealing the structure of the learned dynamics. We consider two types of LCS.
	
	\subsubsection{Forward-Time FTLE (Repelling LCS)}
	By computing the FTLE for the forward-time flow map as described above, we identify \textbf{repelling LCS}. These structures act as dynamical divides or barriers that trajectories tend to move away from. For safety verification, we expect high-FTLE ridges to form protective barriers around obstacles and unsafe regions. Their presence verifies that the agent has learned an effective avoidance policy, treating these areas as "keep-out" zones.
	
	\subsubsection{Backward-Time FTLE (Attracting LCS)}
	Conversely, \textbf{attracting LCS} function as dynamical "highways" or collectors that attract and guide nearby trajectories. Direct computation of these structures would require inverting the flow map, which is generally infeasible for a learned RL policy. Instead, we can approximate the locations of the system's attractors by simulating a large number of trajectories from random initial positions and creating a 2D histogram of their final states after time $T_{\text{int}}$. High-density regions in this histogram correspond to attractors. For robustness, we expect to find a single, dominant attractor at the goal state. The presence of other strong attractors could signify policy vulnerabilities, such as a tendency to get stuck in local optima (trap states) or a failure to converge reliably to the desired goal.

\section{Local Stability Guarantees}
In regions of low FTLE, we can provide a formal guarantee that initially close trajectories remain close, certifying local robustness.

\begin{proposition}[Bounded Divergence]
Let $R$ be a convex region in the state space, and let $\sigma_{\max} = \sup_{s \in R} \sigma(s, T_{\text{int}})$. For any two initial states $s_a, s_b \in R$, their separation after time $T_{\text{int}}$ is bounded by:
\begin{equation}
    \norm{ \flowmap(s_b) - \flowmap(s_a) } \le e^{(\sigma_{\max} \cdot T_{\text{int}})} \cdot \norm{ s_b - s_a }.
\end{equation}
\end{proposition}
\begin{proof}
By the Mean Value Theorem for vector-valued functions, for any $s_a, s_b \in R$, there exists a point $c$ on the line segment between them such that $\norm{\flowmap(s_b) - \flowmap(s_a)} \le \norm{\Jacobian(c)}_2 \cdot \norm{s_b - s_a}$, where $\norm{\Jacobian(c)}_2$ is the induced 2-norm of the Jacobian. The norm is defined as the largest singular value of $\Jacobian(c)$, which is $\sqrt{\lambda_{\max}(\Jacobian(c)\tr \Jacobian(c))} = \sqrt{\lambda_{\max}(\CauchyGreen(c))}$. From the definition of FTLE in Eq. \eqref{eq:ftle}, we have $\sqrt{\lambda_{\max}(\CauchyGreen(c))} = e^{(\sigma(c, T_{\text{int}}) \cdot T_{\text{int}})}$. Combining these facts yields:
\begin{equation*}
    \norm{ \flowmap(s_b) - \flowmap(s_a) } \le e^{(\sigma(c, T_{\text{int}}) \cdot T_{\text{int}})} \cdot \norm{ s_b - s_a }.
\end{equation*}
Since $c \in R$, $\sigma(c, T_{\text{int}}) \le \sigma_{\max}$. The final bound follows directly.
\end{proof}

This proposition allows for a formal $(\delta, \epsilon)$ certificate of robustness. For a desired final separation tolerance $\epsilon$, we can guarantee that any initial perturbation smaller than $\delta$ will be tolerated, where $\delta < \epsilon \cdot \exp(-\sigma_{\max} \cdot T_{\text{int}})$.

	\section{Quantitative Metrics and Algorithms}
	While visualizing Lagrangian Coherent Structures provides a rich, qualitative understanding of policy behavior, a rigorous comparison requires objective, scalar metrics. We introduce three novel metrics derived from the FTLE and trajectory analysis.
	
	\subsection{Metric 1: Mean Boundary Repulsion (MBR)}
	\textbf{Purpose:} To measure the policy's commitment to avoiding danger. This metric directly quantifies the strength of the repelling barriers the agent erects around known unsafe regions. A high score indicates a wide margin of safety.
	
	\subsubsection{Formulation}
	Let $\sspace$ be the total state space and $O \subset \sspace$ be the set of all unsafe states. We define the \textbf{obstacle boundary}, $\partial O$, as the set of all safe states immediately adjacent to an unsafe state:
	\begin{equation}
		\partial O = \{ s \in \sspace \setminus O \mid \exists s' \in O \text{ such that } \norm{s - s'}_1 = 1 \}.
	\end{equation}
	The Mean Boundary Repulsion ($\mbr$) is the arithmetic mean of the forward-time FTLE values on this boundary:
	\begin{equation}
		\mbr = \frac{1}{|\partial O|} \sum_{s_0 \in \partial O} \FTLE(s_0, T_{\text{int}}).
	\end{equation}
	A high $\mbr$ indicates strong, robust avoidance; a low $\mbr$ suggests the policy is vulnerable to perturbations near obstacles.
	
	\subsection{Metric 2: Aggregated Spurious Attractor Strength (ASAS)}
	\textbf{Purpose:} To measure the policy's focus on the goal. This metric quantifies the pull of unintended attractors relative to the pull of the true goal. A low score indicates a robust policy.
	
	\subsubsection{Formulation}
	Let $h(s)$ be the final state density map (histogram) from trajectory simulations and $G \subset \sspace$ be the goal region.
	\begin{enumerate}
		\item \textbf{Benchmark Goal Attraction:} $h_{\text{goal}} = \max_{s \in G} h(s)$.
		\item \textbf{Identify Spurious Peaks:} Find all local peaks in $h(s)$ outside the goal region, $\mathcal{P}_{\text{spurious}}$.
		\item \textbf{Filter for Significance:} Create a set of significant peaks, $\mathcal{P}_{\text{sig}} = \{ s \in \mathcal{P}_{\text{spurious}} \mid h(s) \geq \alpha \cdot h_{\text{goal}} \}$, using a threshold $\alpha \in (0, 1]$.
		\item \textbf{Aggregate Spurious Pull:} Sum the densities of all significant peaks: $h_{\text{agg\_spurious}} = \sum_{s \in \mathcal{P}_{\text{sig}}} h(s)$.
		\item \textbf{Compute ASAS Ratio:} $\asas = h_{\text{agg\_spurious}} / h_{\text{goal}}$. If $h_{\text{goal}}=0$, $\asas \defeq \infty$.
	\end{enumerate}
	An $\asas \approx 0$ indicates ideal robustness, while $\asas \geq 1$ signals a critical lack of robustness where spurious attractors dominate the goal.
	
	\subsection{Metric 3: Temporally-Aware Spurious Attractor Strength (TASAS)}
	\textbf{Purpose:} To refine ASAS by distinguishing between transient "highways" and terminal "traps."
	
	\subsubsection{Formulation}
	TASAS weights each spurious attractor's strength by a \textbf{persistence factor}, $\rho(p)$. For each significant peak $p \in \mathcal{P}_{\text{sig}}$, we run new simulations starting from $p$ and measure the fraction that reach the goal, known as the escape ratio $E(p)$. The persistence factor is $\rho(p) = 1 - E(p)$.
	
	The persistence-weighted spurious strength is:
	\begin{equation}
		h_{\text{tasas}} = \sum_{p \in \mathcal{P}_{\text{sig}}} h(p) \cdot \rho(p).
	\end{equation}
	The final TASAS ratio is:
	\begin{equation}
		\tasas = \frac{h_{\text{tasas}}}{h_{\text{goal}}}.
	\end{equation}
	A non-zero $\tasas$ confirms the existence of at least one genuine terminal trap region, providing a more accurate assessment of policy failure than ASAS alone.
	
	\subsection{Algorithm Implementations and Complexity}
	Here we present the pseudocode for our methods and analyze their computational complexity. $|S|$ is the number of valid states in the grid.
	
	\subsubsection{FTLE Calculation}
	\paragraph{Complexity Analysis} The dominant step is the flow map calculation. For each of the $|S|$ states, the algorithm simulates the policy forward for $T_{\text{int}}$ steps. The total time complexity is $\boldsymbol{O(|S| \cdot T_{\textbf{int}})}$.
	
	\begin{algorithm}[!t]
		\caption{FTLE Calculation for a Grid-Based System}
		\label{alg:ftle_calculation_condensed}
		\begin{algorithmic}[1]
			\Procedure{CalculateFTLE}{Grid $G$, Obstacles $O$, Policy $\pi$, Horizon $T_{int}$}
			\State \textbf{Input:} Grid $G$, Obstacles $O$, Policy $\pi$, Horizon $T_{int}$.
			\State \textbf{Output:} FTLE field $\sigma$.
			\Statex \Comment{\textit{Step 1: Compute Flow Map $\Phi$}}
			\State Initialize map $\Phi$; $f(s) \gets \text{Transition}(s, \pi(s))$
			\For{each state $s_0 \in (G \setminus O)$}
			\State $s_{current} \gets s_0$
			\For{$k=1$ to $T_{int}$}
			\State $s_{current} \gets f(s_{current})$
			\EndFor
			\State $\Phi(s_0) \gets s_{current}$
			\EndFor
			\Statex \Comment{\textit{Step 2: Compute FTLE field $\sigma$}}
			\State Initialize $\sigma$ as a zero matrix.
			\For{each state $s_0 = (r,c) \in (G \setminus O)$}
			\State $\boldsymbol{\phi}_{s0} \gets \Phi(s_0)$
			\State $\boldsymbol{u_1} \gets \Phi(s_0 + (0, 1)) - \boldsymbol{\phi}_{s0}$
			\State $\boldsymbol{u_2} \gets \Phi(s_0 + (1, 0)) - \boldsymbol{\phi}_{s0}$
			\State $\boldsymbol{J} \gets [\boldsymbol{u_1}^T, \boldsymbol{u_2}^T]$
			\State $\boldsymbol{C} \gets \boldsymbol{J}^T \boldsymbol{J}$
			\State $\lambda_{\max} \gets \max(\text{eigenvalues}(\boldsymbol{C}))$
			\If{$\lambda_{\max} > 0$}
			\State $\sigma(s_0) \gets \frac{1}{2 \cdot T_{int}} \ln(\lambda_{\max})$
			\EndIf
			\EndFor
			\State \textbf{return} $\sigma$
			\EndProcedure
		\end{algorithmic}
	\end{algorithm}
	
	\subsubsection{Mean Boundary Repulsion (MBR)}
	\paragraph{Complexity Analysis} Identifying the boundary requires iterating through all $|S|$ valid states and checking a constant number of neighbors. The complexity is $\boldsymbol{O(|S|)}$.
	
	\begin{algorithm}[!b]
		\caption{Mean Boundary Repulsion (MBR)}
		\label{alg:mbr}
		\begin{algorithmic}[1]
			\Procedure{CalculateMBR}{$\sigma$, Grid $G$, Obstacles $O$}
			\State Initialize boundary set $\partial O \gets \emptyset$
			\State Define neighbors $D \gets \{(0,1), (0,-1), (1,0), (-1,0)\}$
			\For{each state $s \in (G \setminus O)$}
			\For{each direction $d \in D$}
			\If{$(s + d) \in O$}
			\State Add $s$ to $\partial O$; \textbf{break}
			\EndIf
			\EndFor
			\EndFor
			\If{$|\partial O| = 0$} \textbf{return} $0$ \EndIf
			\State $MBR \gets \frac{1}{|\partial O|} \sum_{s \in \partial O} \sigma(s)$
			\State \textbf{return} $MBR$
			\EndProcedure
		\end{algorithmic}
	\end{algorithm}
	
	\subsubsection{Aggregated Spurious Attractor Strength (ASAS)}
	\paragraph{Complexity Analysis} The main work involves iterating through all $|S|$ states to find local peaks in the attractor field $h$. This is a constant-time check per state, so complexity is $\boldsymbol{O(|S|)}$.
	
	\begin{algorithm}[!t]
		\caption{Aggregated Spurious Attractor Strength (ASAS)}
		\label{alg:asas}
		\begin{algorithmic}[1]
			\Procedure{CalculateASAS}{$h$, $G_{goal}$, $\alpha$}
			\State $h_{goal} \gets \max_{s \in G_{goal}} h(s)$
			\If{$h_{goal} = 0$} \textbf{return} $\infty$ \EndIf
			\State $P_{\text{spurious}} \gets \emptyset$; $P_{\text{sig}} \gets \emptyset$
			\For{each state $s \in (G \setminus G_{goal})$}
			\If{$h(s)$ is a local maximum} Add $s$ to $P_{\text{spurious}}$ \EndIf
			\EndFor
			\For{each peak $p \in P_{\text{spurious}}$}
			\If{$h(p) \ge \alpha \cdot h_{goal}$} Add $p$ to $P_{\text{sig}}$ \EndIf
			\EndFor
			\State $h_{\text{agg\_spurious}} \gets \sum_{p \in P_{\text{sig}}} h(p)$
			\State \textbf{return} $h_{\text{agg\_spurious}} / h_{goal}$, $P_{\text{sig}}$
			\EndProcedure
		\end{algorithmic}
	\end{algorithm}
	
	\subsubsection{Temporally-Aware Spurious Attractor Strength (TASAS)}
	\paragraph{Complexity Analysis} This metric iterates through each significant spurious peak, $|P_{\text{sig}}|$, and runs $N_{\text{sim}}$ simulations for $T_{\text{escape}}$ steps. The complexity is $\boldsymbol{O(|P_{\textbf{sig}}| \cdot N_{\textbf{sim}} \cdot T_{\textbf{escape}})}$.
	
	\begin{algorithm}[!t]
		\caption{Temporally-Aware Spurious Attractor Strength (TASAS)}
		\label{alg:tasas}
		\begin{algorithmic}[1]
			\Procedure{CalculateTASAS}{$h, P_{\text{sig}}, h_{goal}, \pi, N, T_{esc}$}
			\If{$h_{goal} = 0$} \textbf{return} $\infty$ \EndIf
			\State $h_{\text{tasas}} \gets 0$
			\For{each peak $p \in P_{\text{sig}}$}
			\State $escaped \gets 0$
			\For{$i=1$ to $N$}
			\State $s \gets p$
			\For{$k=1$ to $T_{esc}$}
			\State $s \gets \text{Transition}(s, \pi(s))$
			\If{$s$ is in Goal}
			\State $escaped \gets escaped + 1$; \textbf{break}
			\EndIf
			\EndFor
			\EndFor
			\State $E(p) \gets escaped / N$; $\rho(p) \gets 1 - E(p)$
			\State $h_{\text{tasas}} \gets h_{\text{tasas}} + h(p) \cdot \rho(p)$
			\EndFor
			\State \textbf{return} $h_{\text{tasas}} / h_{goal}$
			\EndProcedure
		\end{algorithmic}
	\end{algorithm}
	
	\section{Agent Architecture and Training Protocol}
	\label{sec:agent_training}
	
	To ensure a consistent and fair comparison, we use a single, standardized Deep Q-Network (DQN) agent configuration.
	
	\subsection{Q-Network Architecture}
	
	\begin{figure}[!t]
		\centering
		\includegraphics[width=\columnwidth]{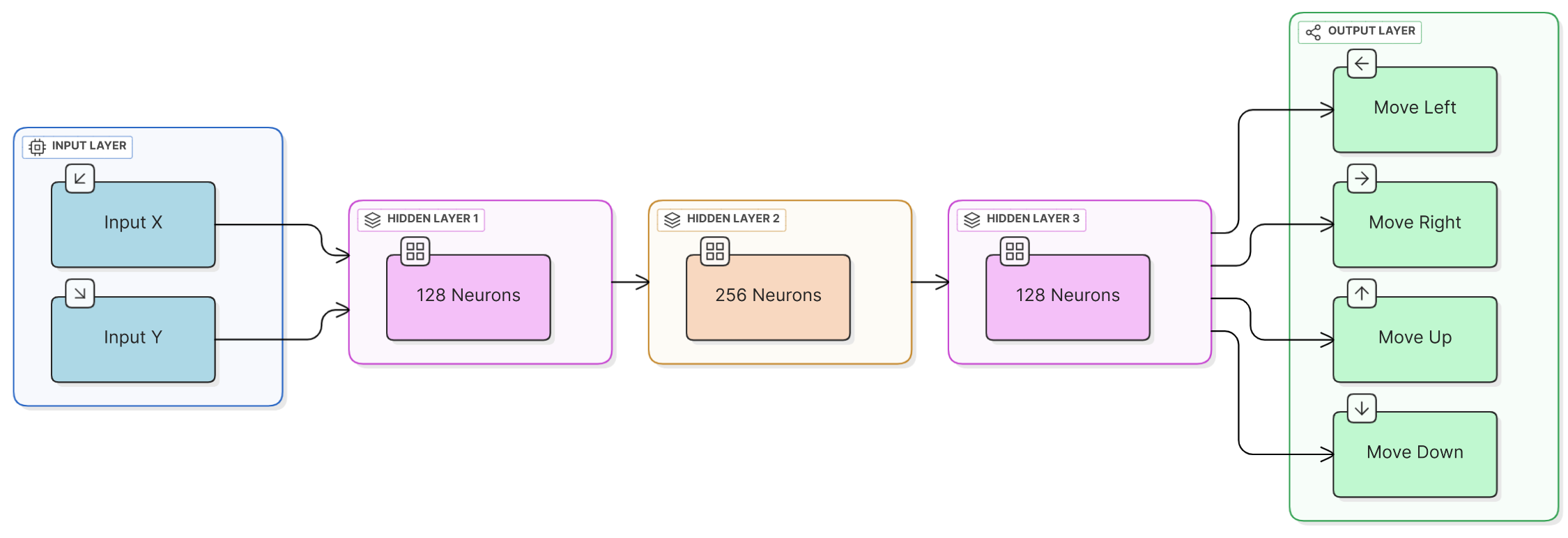}
		\caption{Neural Network Architecture}
		\label{fig:nn_arch}
	\end{figure}
	
	The action-value function, $Q(s, a)$, is approximated using a feed-forward neural network in PyTorch.
	\begin{itemize}
		\item \textbf{Input Layer:} A 2D state vector $(x, y)$.
		\item \textbf{Hidden Layers:} Three hidden layers with 128, 256, and 128 neurons, respectively, each followed by a ReLU activation.
		\item \textbf{Output Layer:} A linear layer producing Q-values for each discrete action.
	\end{itemize}
	
	\subsection{Training Details}
	The agent is trained using the standard DQN algorithm with experience replay and a fixed Q-target.
	\begin{itemize}
		\item \textbf{Replay Memory:} Capacity of 10,000 transitions.
		\item \textbf{Optimization:} Adam optimizer with learning rate $1 \times 10^{-4}$ and Mean Squared Error loss.
		\item \textbf{Target Network:} Weights updated periodically from the policy network.
		\item \textbf{Hyperparameters:} Discount factor $\gamma = 0.99$.
		\item \textbf{Exploration:} An $\epsilon$-greedy strategy, with $\epsilon$ annealed over time. All policy evaluations in this paper are performed in a purely greedy mode ($\epsilon=0$).
	\end{itemize}
	The same protocol is used across all environments to isolate the effect of environmental complexity.

	\section{Experiments and Results}
	\label{sec:experiments}
	
	We validate our framework using a DQN agent trained in three 2D grid-worlds: a \textbf{Simple Wall}, \textbf{Scattered Blocks}, and a \textbf{U-Shape Trap}. We analyze the learned policy at episodes 0, 150, 750, and 1200, generating FTLE fields, attractor plots, and our quantitative metrics (MBR, ASAS, TASAS).
	
	\subsection{Environment 1: Simple Wall}
	This baseline environment tests fundamental avoidance behavior.

    \begin{figure*}[!h]
	\centering
	
	\begin{subfigure}{0.88\textwidth}
    \centering
		\includegraphics[width=\linewidth, height=0.25\textheight, keepaspectratio]{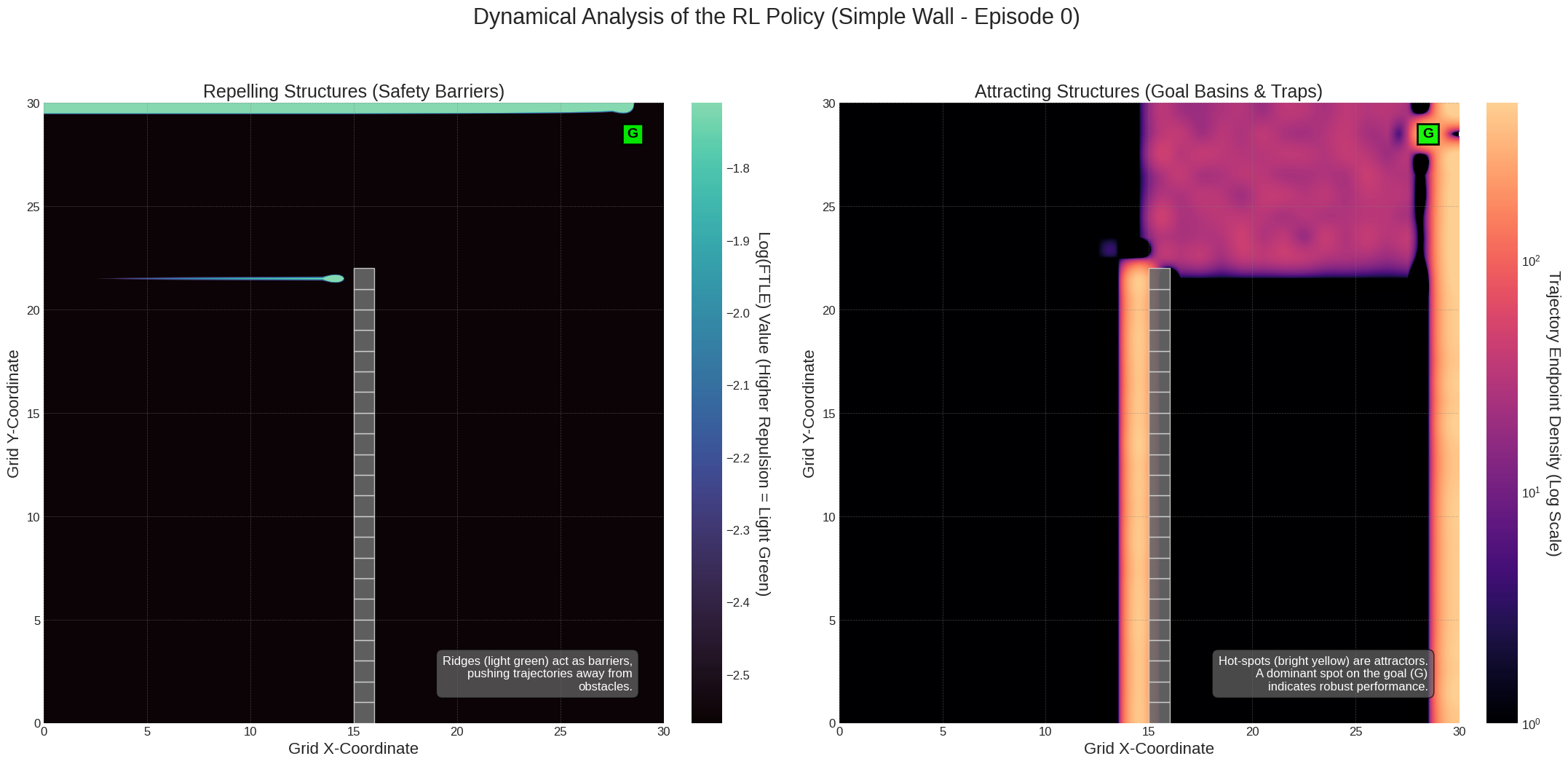}
		\caption{Attraction and repulsion fields in Simple Wall environment at 0 Episodes}
	\end{subfigure}%

	\begin{subfigure}{0.88\textwidth}
    \centering
		\includegraphics[width=\linewidth, height=0.25\textheight, keepaspectratio]{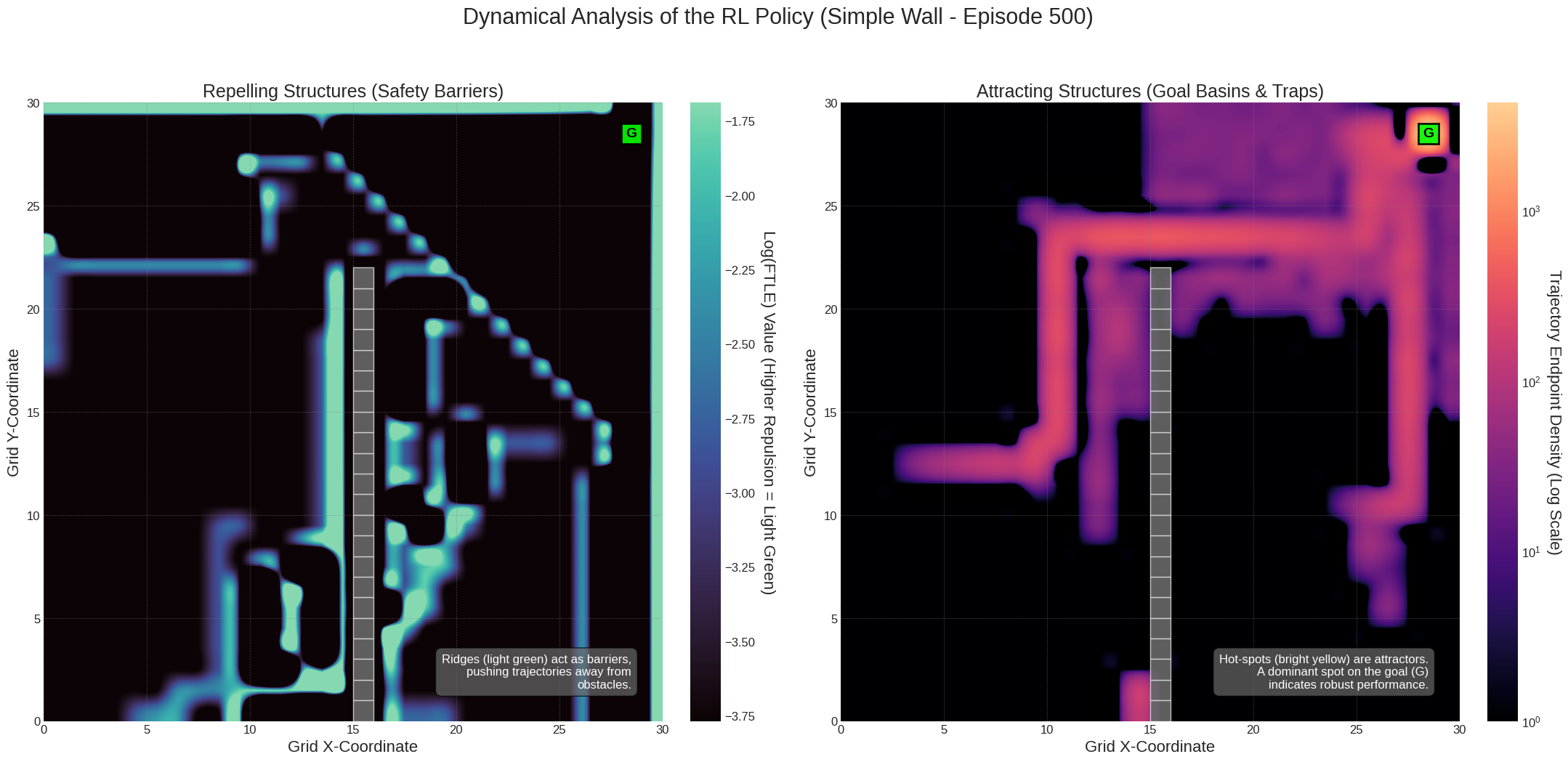}
		\caption{Attraction and repulsion fields in Simple Wall environment at 500 Episodes}
	\end{subfigure}%

	\begin{subfigure}{0.88\textwidth}
	\centering
        \includegraphics[width=\linewidth, height=0.25\textheight, keepaspectratio]{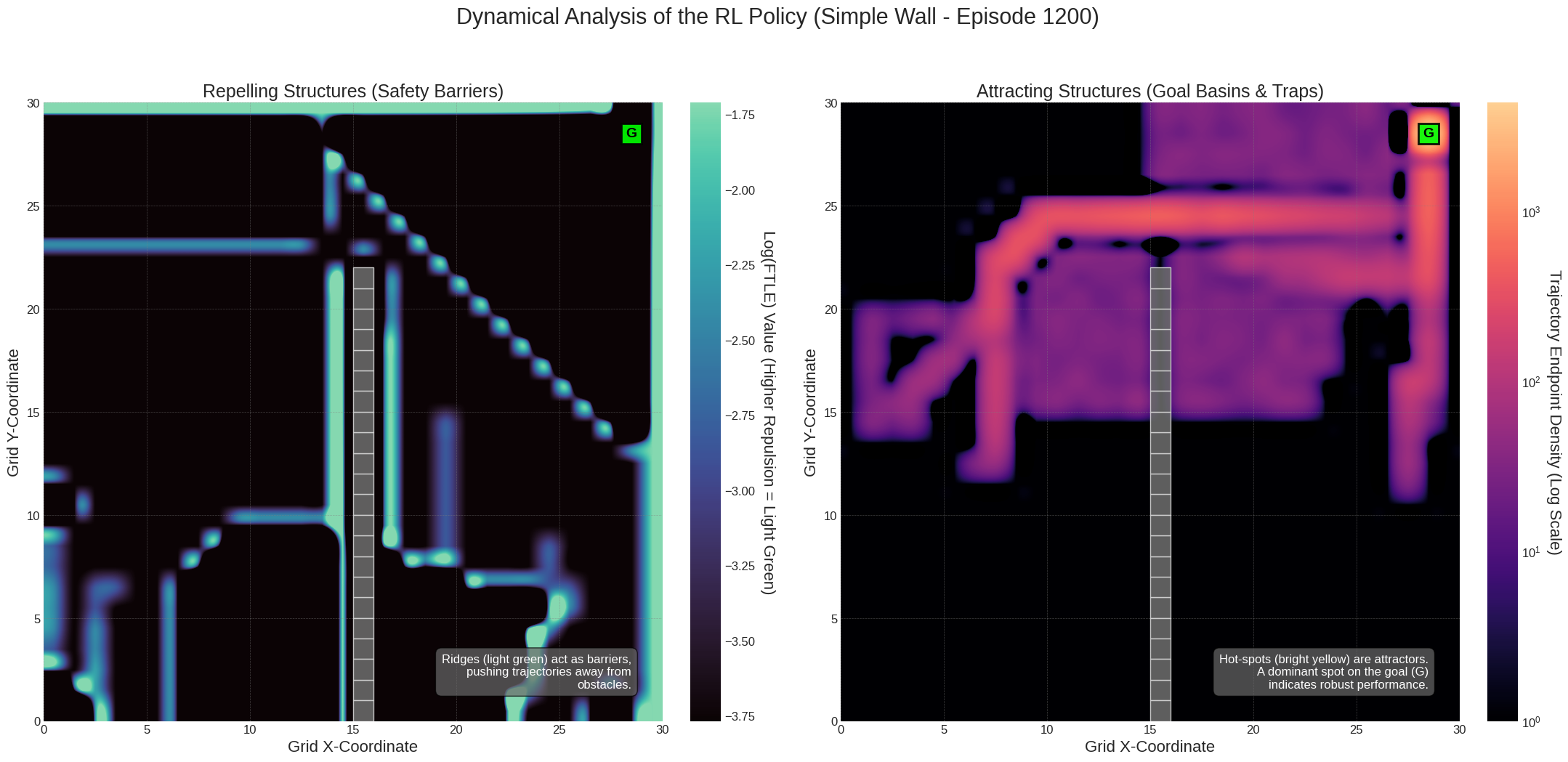}
		\caption{Attraction and repulsion fields in Simple Wall environment at 1200 Episodes}
	\end{subfigure}%
	\hfill

	\caption{Evolution of dynamical structures for the Simple Wall environment. The untrained agent (a) is chaotic. As training progresses (b-c), a distinct repelling LCS (bright ridge) forms along the wall, and the attractor field consolidates around the goal (G).}
		\label{fig:simple_wall_evolution}

    \end{figure*}
	
	\textbf{Visual Analysis:} As shown in Figure \ref{fig:simple_wall_evolution}, the policy evolves from chaos (a) to a highly structured state. By episode 1200 (d), a prominent high-FTLE ridge acts as a safety barrier along the wall, while the attractor plot shows a clear "highway" guiding the agent to the goal.
	
	\textbf{Quantitative Analysis:} Table \ref{tab:metrics_sw} confirms this. The MBR score climbs to a robust 0.0994, quantifying the strong safety barrier. The ASAS and TASAS scores plummet to near-zero, certifying that the mature policy is free of persistent traps.
	
	\begin{table}[!h]
		\centering
		\caption{Metrics for the Simple Wall Environment.}
		\label{tab:metrics_sw}
		\begin{tabular}{@{}cccc@{}}
			\toprule
			\textbf{Episode} & \textbf{MBR} & \textbf{ASAS} & \textbf{TASAS} \\ \midrule
			0                & 0.0040       & 21.1217       & 21.0317        \\
			50               & 0.1029       & 0.4961        & 0.0577         \\
			150              & 0.0901       & 0.7437        & 0.1519         \\
			750              & 0.1058       & 0.7496        & 0.2028         \\
			1200             & 0.0994       & 0.4427        & 0.0544         \\ \bottomrule
		\end{tabular}
	\end{table}

	\subsection{Environment 2: Scattered Blocks}
	This environment requires navigation through a more cluttered space.

    \begin{figure*}[!h]
	\centering
	
	\begin{subfigure}{0.85\textwidth}
    \centering
		\includegraphics[width=\linewidth, height=0.25\textheight, keepaspectratio]{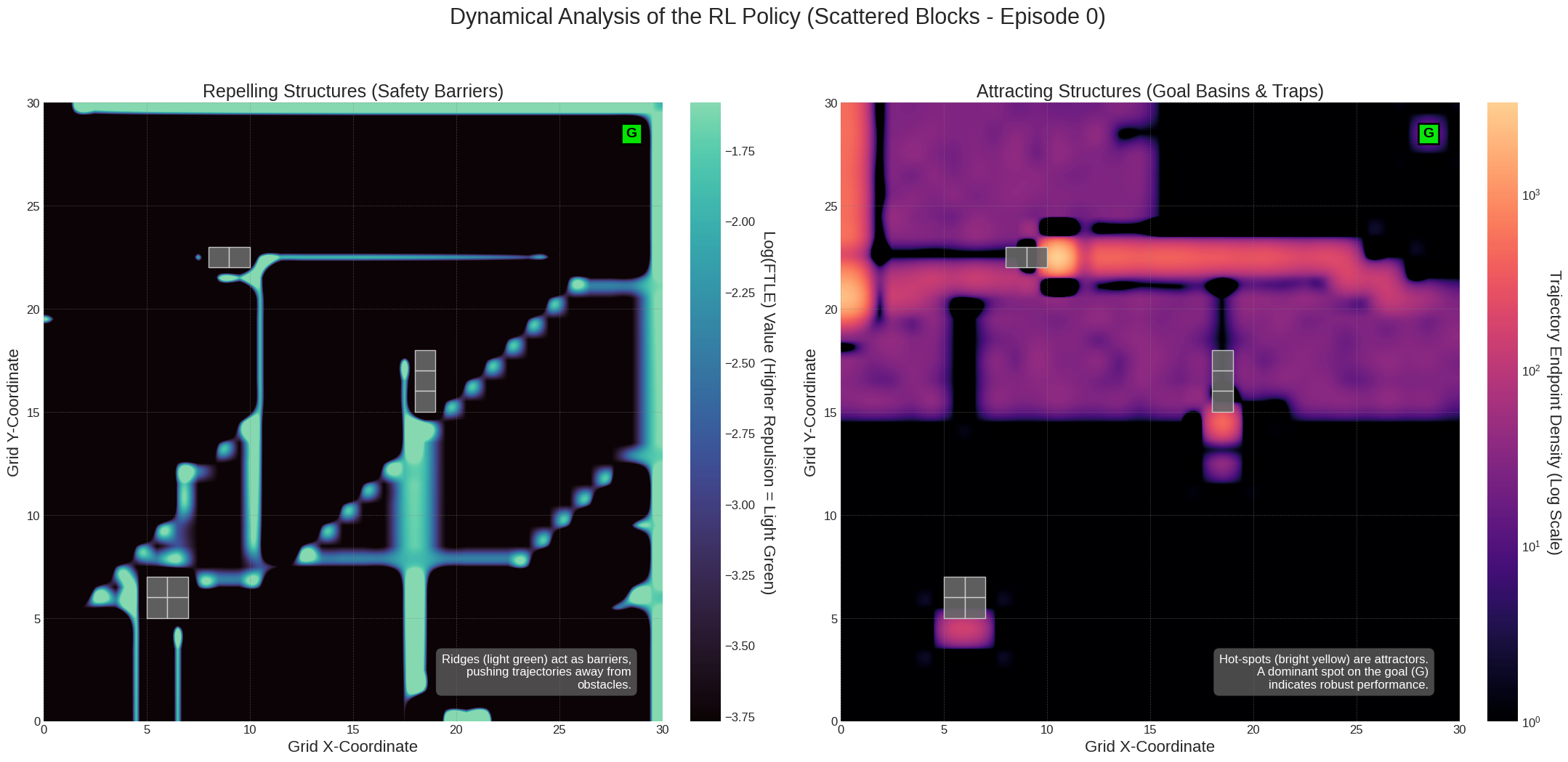}
		\caption{Attraction and repulsion fields in Scattered Blocks environment at 0 Episodes}
	\end{subfigure}%

	\begin{subfigure}{0.85\textwidth}
    \centering
		\includegraphics[width=\linewidth, height=0.25\textheight, keepaspectratio]{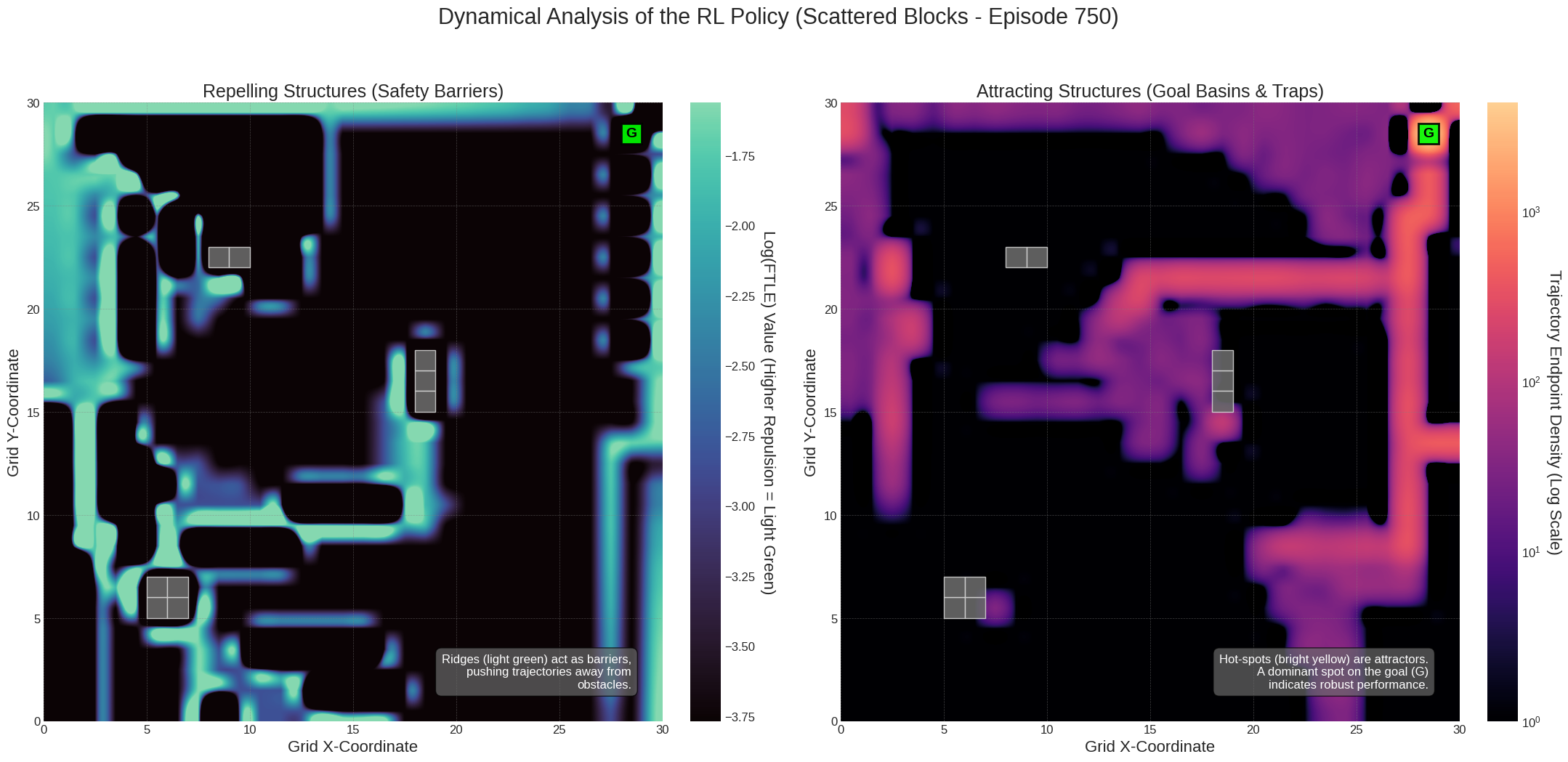}
		\caption{Attraction and repulsion fields in Scattered Blocks environment at 750 Episodes}
	\end{subfigure}%

	\begin{subfigure}{0.85\textwidth}
    \centering
		\includegraphics[width=\linewidth, height=0.25\textheight, keepaspectratio]{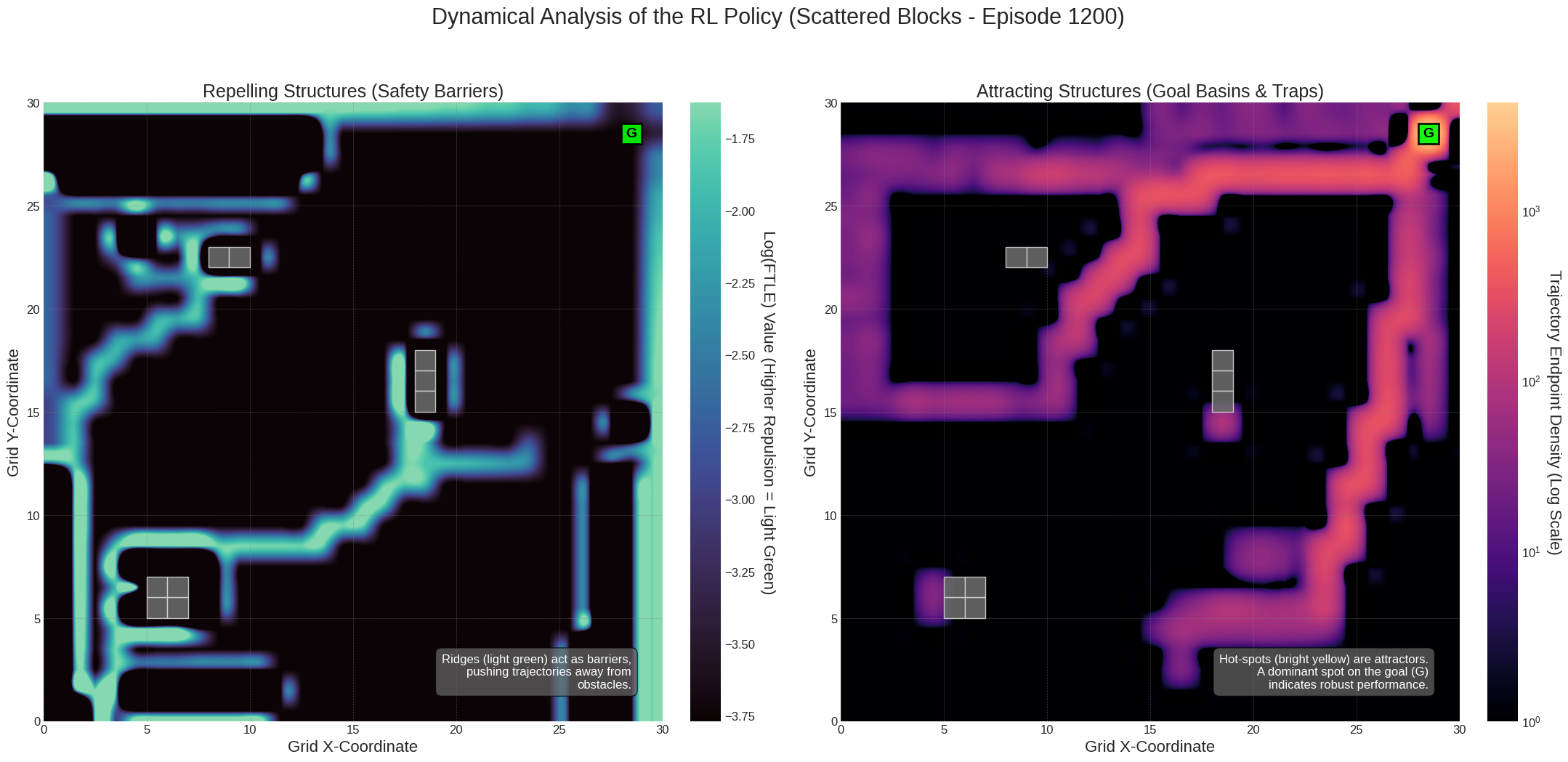}
		\caption{Attraction and repulsion fields in Scattered Blocks environment at 1200 Episodes}
	\end{subfigure}%
	\hfill

	\caption{Dynamical evolution for the Scattered Blocks environment. The agent learns to form localized repelling structures around each obstacle (b-c), creating a safe, winding path to the goal.}
	\label{fig:scattered_blocks_evolution}

    \end{figure*}

	\textbf{Quantitative Analysis:} Table \ref{tab:metrics_sb} shows successful, albeit imperfect, learning. The MBR value stabilizes around a healthy 0.08-0.11. The TASAS value, while low, settles at a non-zero 0.0550. This is a key insight: our framework quantitatively flags that the faint, spurious attractor visible in the final plot (Figure \ref{fig:scattered_blocks_evolution}d) is indeed a minor but persistent trap state.
	
	\begin{table}[H]
		\centering
		\caption{Metrics for the Scattered Blocks Environment.}
		\label{tab:metrics_sb}
		\begin{tabular}{@{}cccc@{}}
			\toprule
			\textbf{Episode} & \textbf{MBR} & \textbf{ASAS} & \textbf{TASAS} \\ \midrule
			0                & 0.0845       & 515.8000      & 515.8000       \\
			150              & 0.0900       & 0.8182        & 0.0501         \\
			750              & 0.1100       & 0.7674        & 0.2244         \\
			1200             & 0.0807       & 0.8311        & 0.0550         \\ \bottomrule
		\end{tabular}
	\end{table}
	
	\subsection{Environment 3: U-Shape Trap}
	This environment is designed to be deceptive, requiring a non-greedy path.
	
        \begin{figure*}[!h]
	\centering
	
	\begin{subfigure}{0.88\textwidth}
    \centering
		\includegraphics[width=\linewidth, height=0.27\textheight, keepaspectratio]{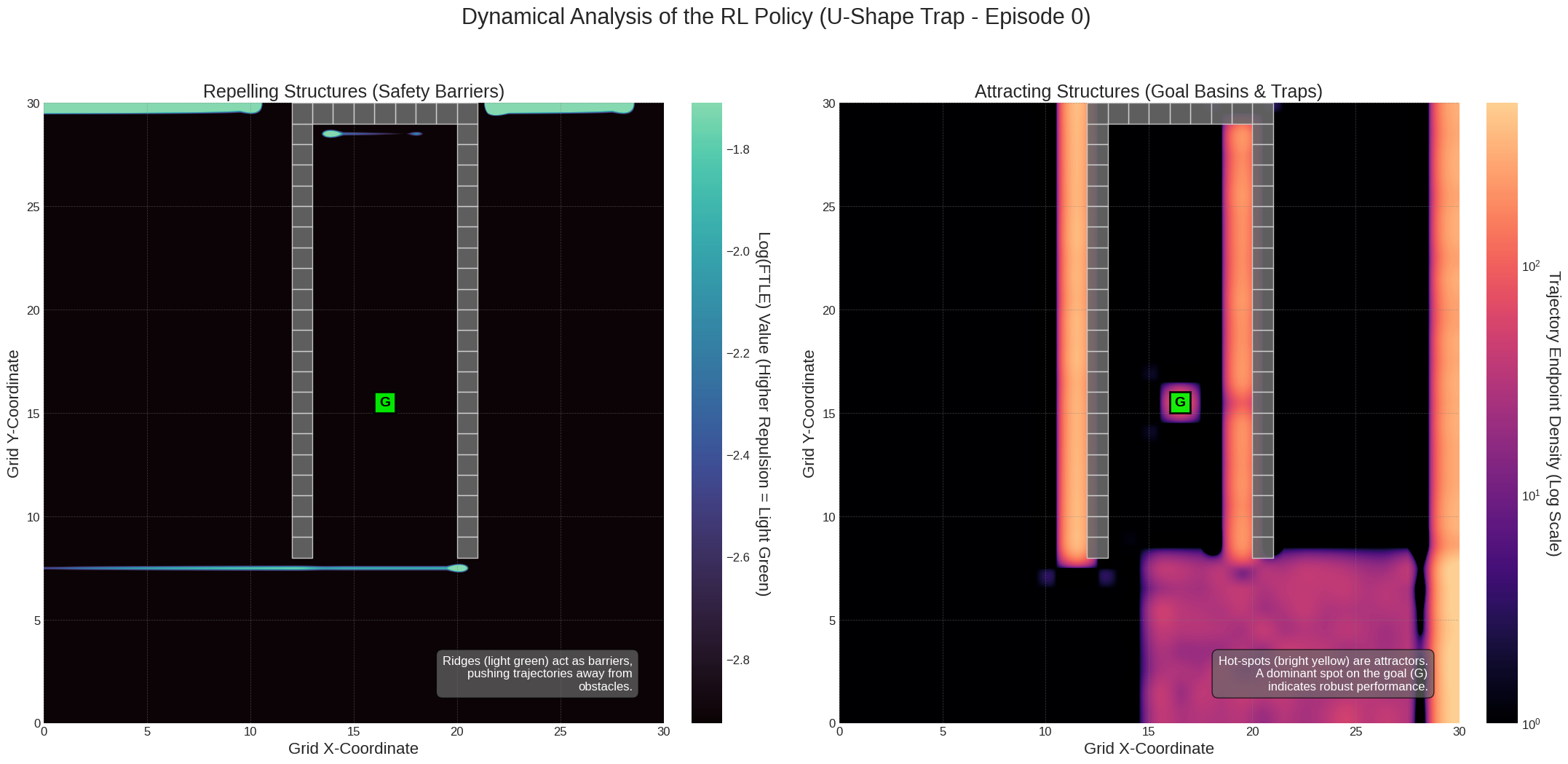}
		\caption{Attraction and repulsion fields in U-Shape Trap environment at 0 Episodes}
	\end{subfigure}%

	\begin{subfigure}{0.88\textwidth}
    \centering
		\includegraphics[width=\linewidth, height=0.27\textheight, keepaspectratio]{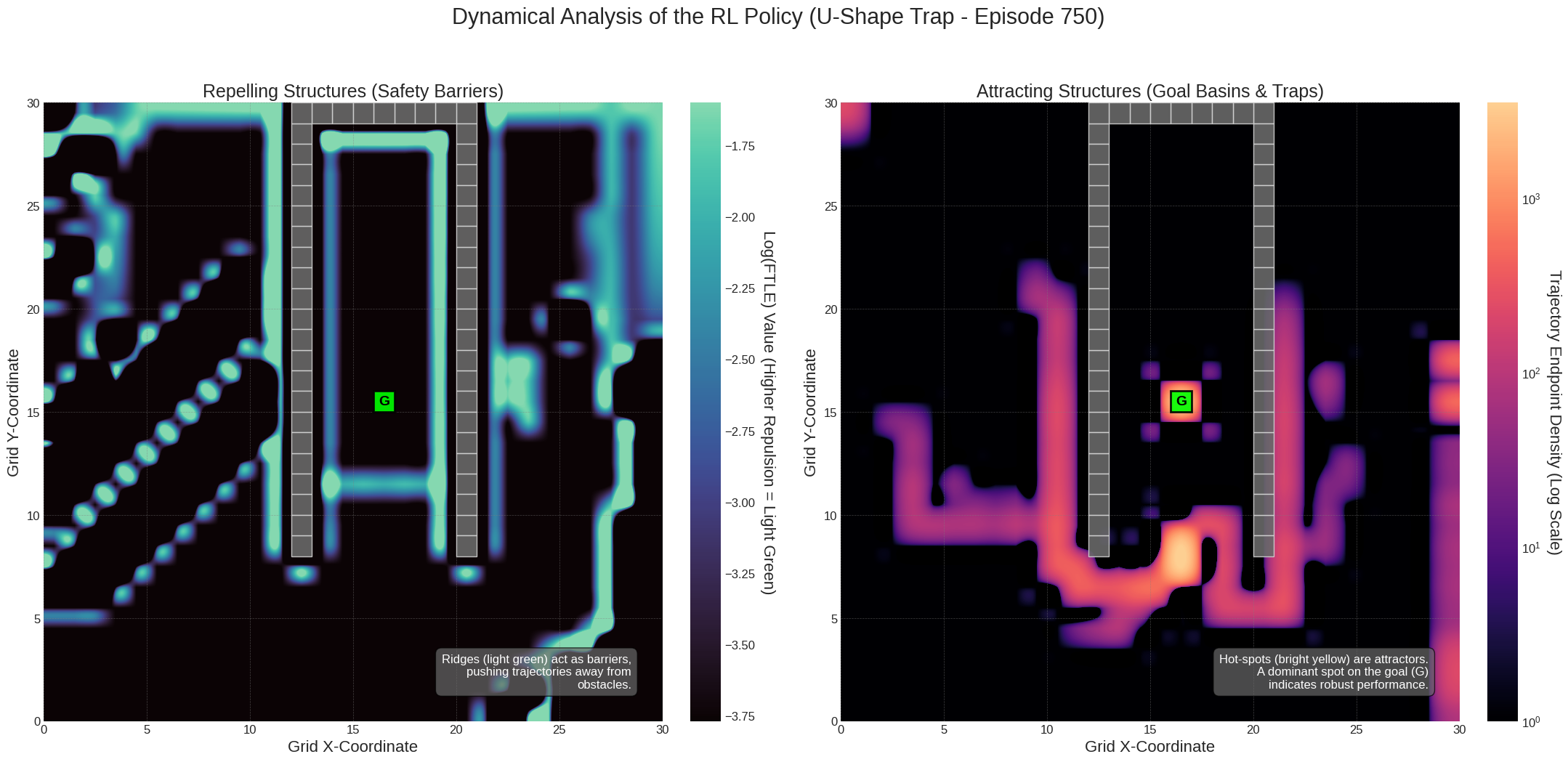}
		\caption{Attraction and repulsion fields in U-Shape Trap environment at 750 Episodes}
	\end{subfigure}%

	\begin{subfigure}{0.88\textwidth}
    \centering
		\includegraphics[width=\linewidth, height=0.27\textheight, keepaspectratio]{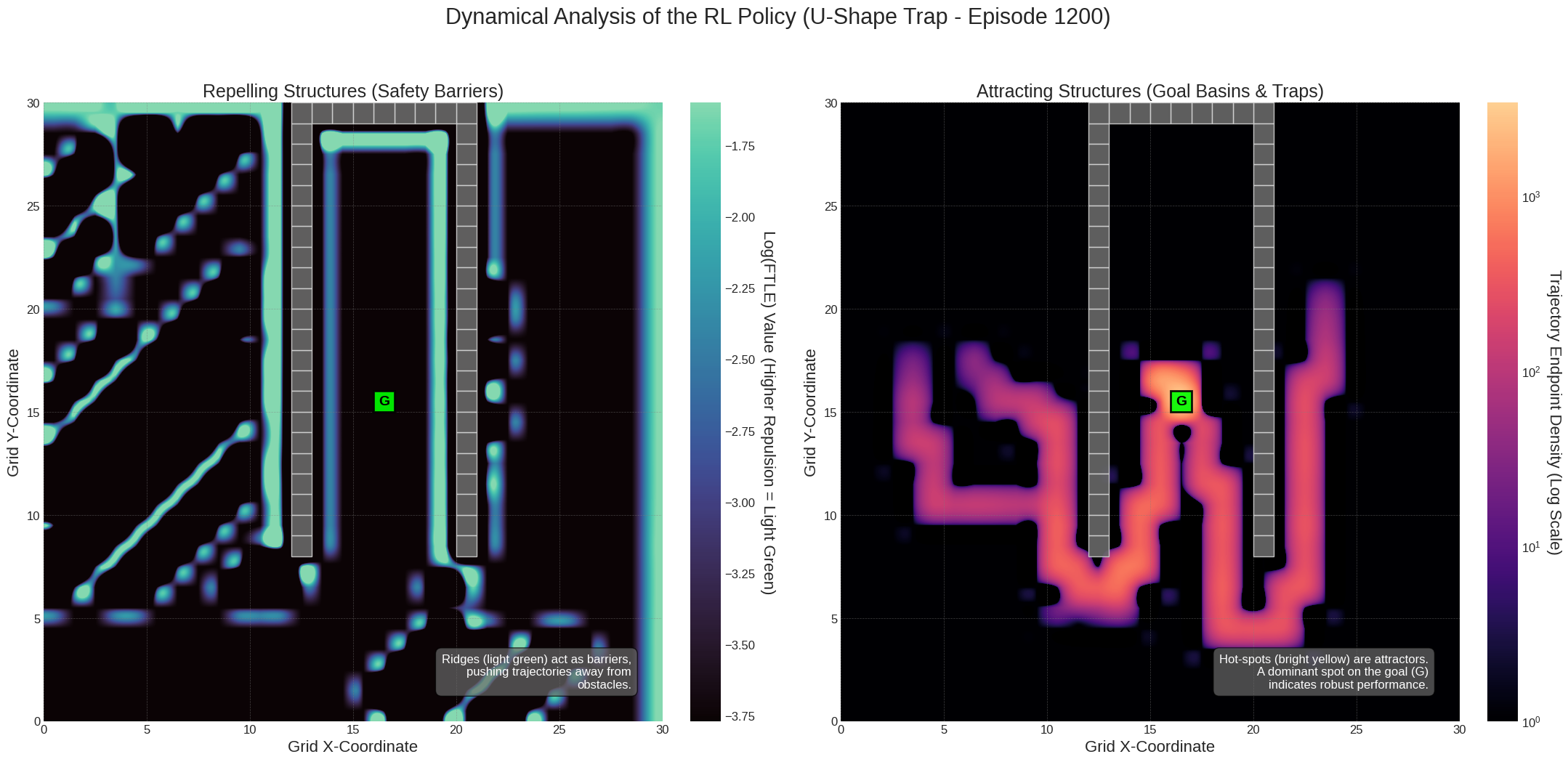}
		\caption{Attraction and repulsion fields in U-Shape Trap environment at 1200 Episodes}
	\end{subfigure}%
	\hfill

	\caption{Dynamical evolution for the U-Shape Trap environment. Visually assessing robustness is difficult, as the agent forms attracting "highways" (b, c) whose persistence is not obvious.}
		\label{fig:u_shape_evolution}

    \end{figure*}
    
	\textbf{Quantitative Analysis:} Our metrics tell a clear story of successful learning (Table \ref{tab:metrics_us}). MBR improves to a strong 0.0911. The evolution of ASAS and TASAS is especially insightful. The ASAS value remains relatively high (1.2258), which might suggest poor robustness. However, the TASAS metric plummets to just 0.0727. This crucial distinction shows that the attracting "highways" are transient waypoints, not terminal traps, powerfully demonstrating the framework's ability to untangle complex behaviors.
	
	\begin{table}[H]
		\centering
		\caption{Metrics for the U-Shape Trap Environment.}
		\label{tab:metrics_us}
		\begin{tabular}{@{}cccc@{}}
			\toprule
			\textbf{Episode} & \textbf{MBR} & \textbf{ASAS} & \textbf{TASAS} \\ \midrule
			0                & 0.0101       & 82.5876       & 82.5876        \\
			150              & 0.0955       & 1.2938        & 1.2938         \\
			750              & 0.0911       & 1.9835        & 1.9835         \\
			1200             & 0.0911       & 1.2258        & 0.0727         \\ \bottomrule
		\end{tabular}
	\end{table}
	
	\textbf{Trajectories in three environments.}
	We have plotted the trajectories on top of the attraction field. Trajectories move towards brighter regions to reach the goal. Figure \ref{fig:grid} shows how trajectories converge onto learned routes that funnel them towards the goal while avoiding obstacles. 

\begin{figure*}[!h]
	\centering
	
	\begin{subfigure}{0.48\textwidth}
		\includegraphics[width=\linewidth, height=0.27\textheight, keepaspectratio]{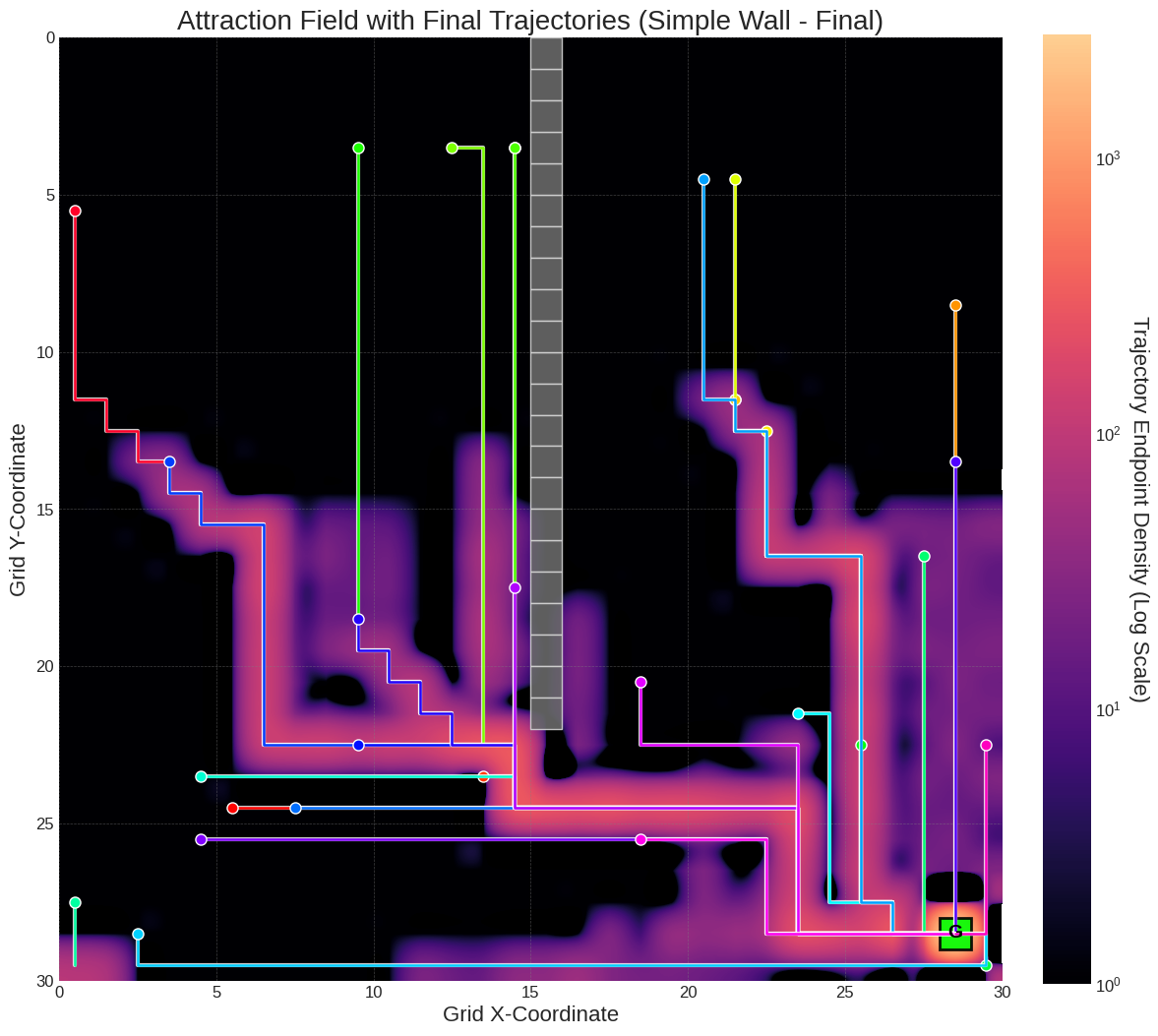}
		\caption{Trajectories and attraction field in simple all}
	\end{subfigure}%
	\hfill
	\begin{subfigure}{0.48\textwidth}
		\includegraphics[width=\linewidth, height=0.27\textheight, keepaspectratio]{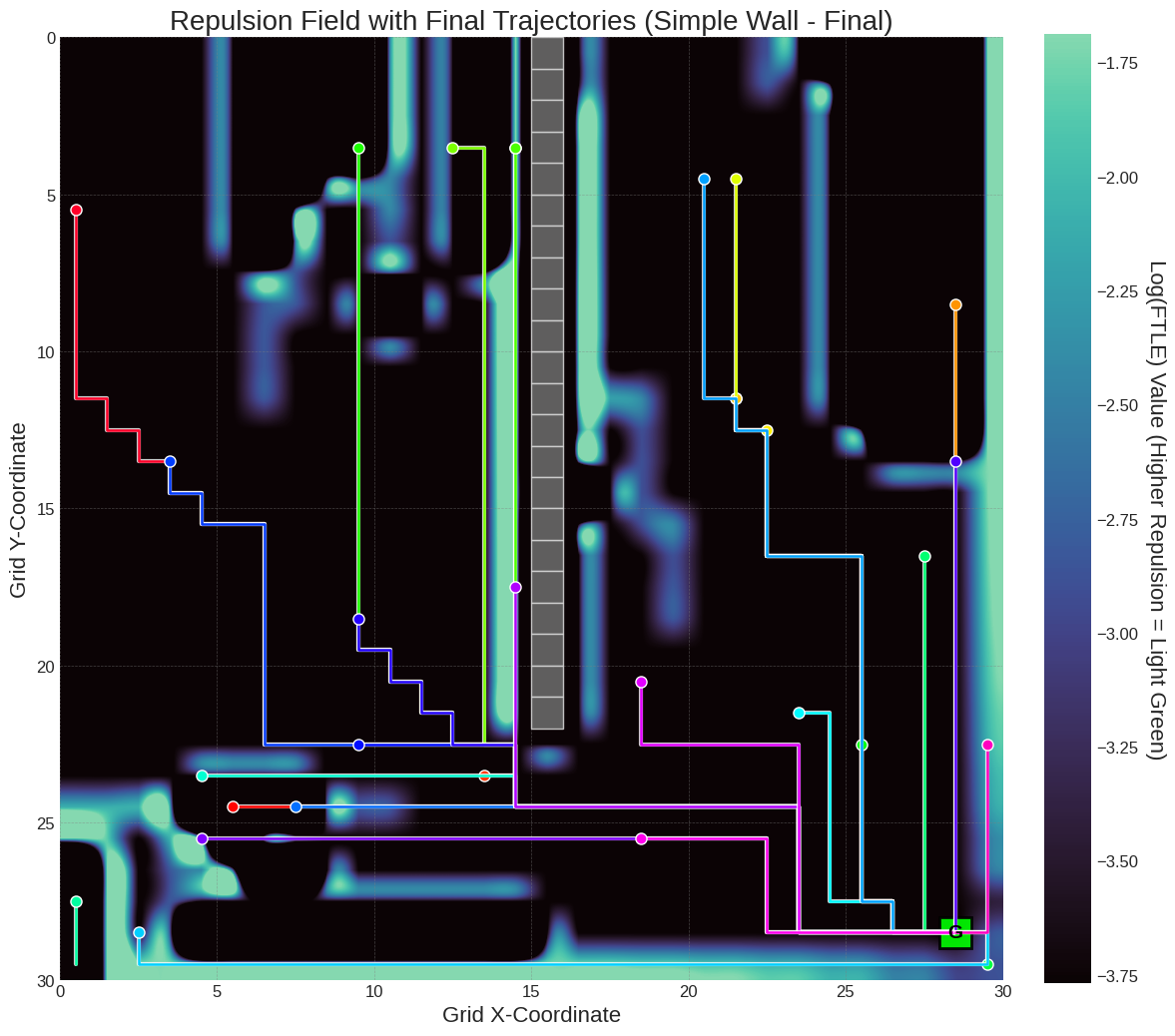}
		\caption{Trajectories and repulsion field in simple wall}
	\end{subfigure}
	
	\begin{subfigure}{0.48\textwidth}
		\includegraphics[width=\linewidth, height=0.27\textheight, keepaspectratio]{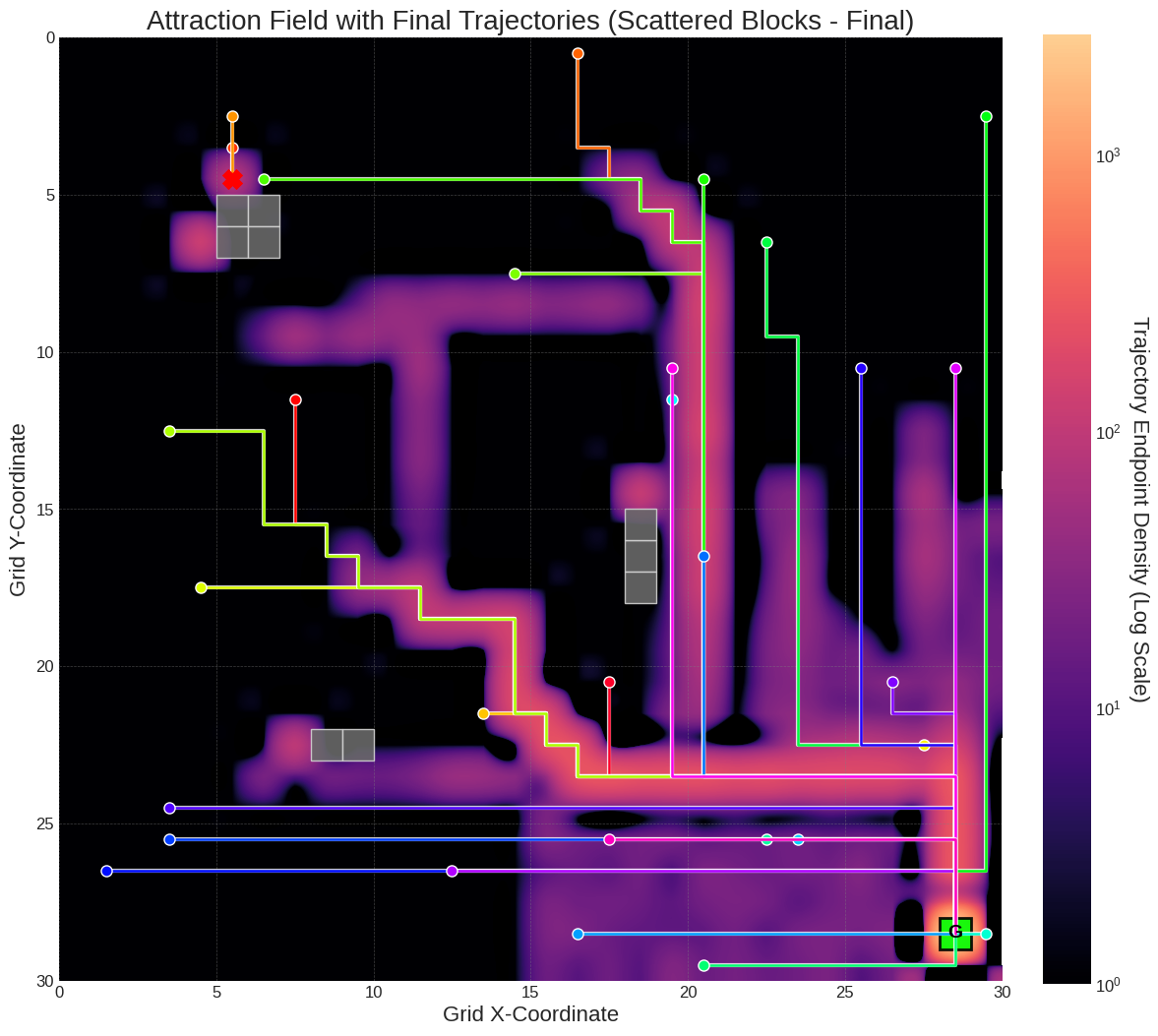}
		\caption{Trajectories and attraction field in scattered blocks environment}
	\end{subfigure}%
	\hfill
	\begin{subfigure}{0.48\textwidth}
		\includegraphics[width=\linewidth, height=0.27\textheight, keepaspectratio]{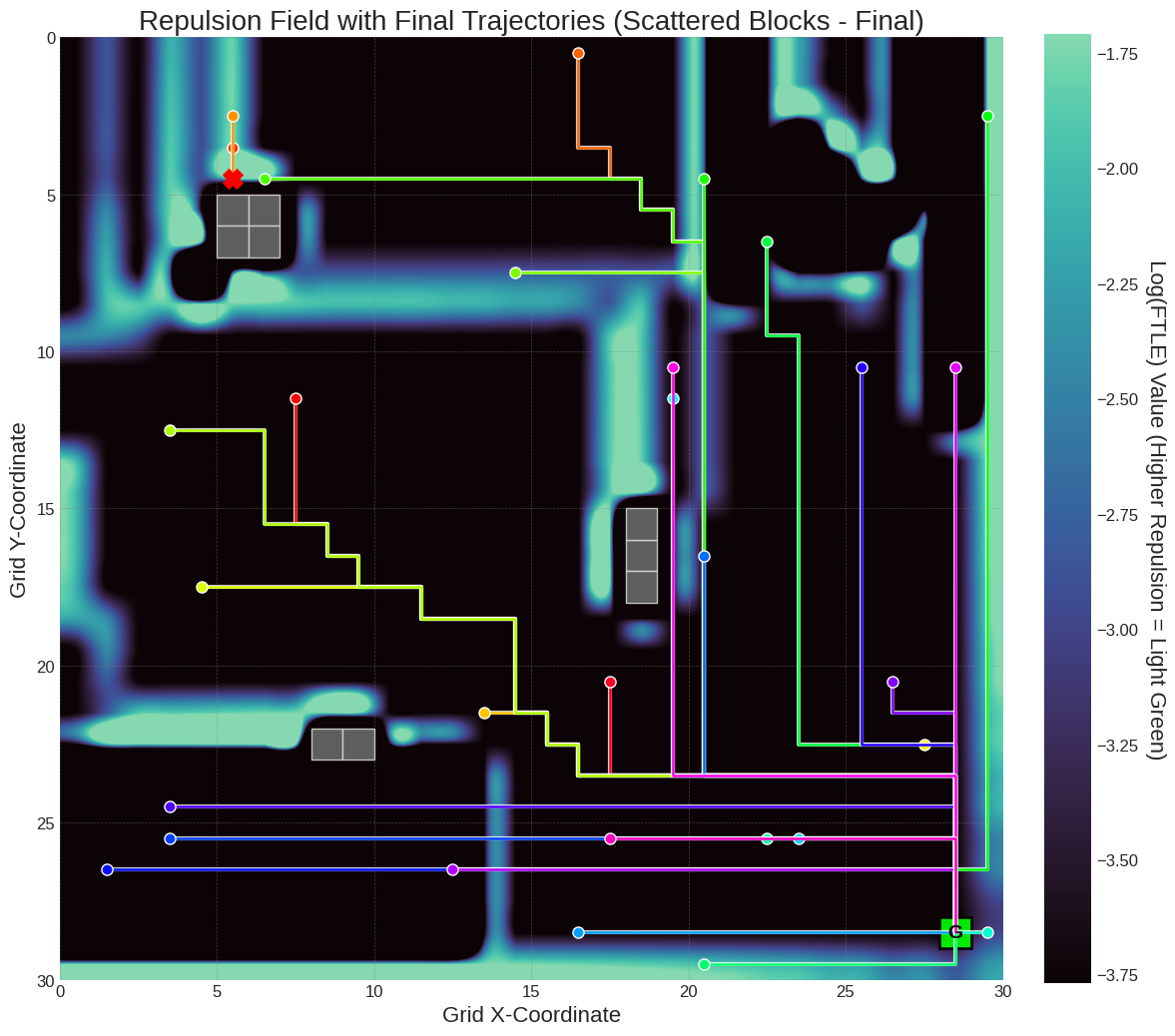}
		\caption{Trajectories and repulsion field in scattered blocks environment}
	\end{subfigure}
	
	\begin{subfigure}{0.48\textwidth}
		\includegraphics[width=\linewidth, height=0.27\textheight, keepaspectratio]{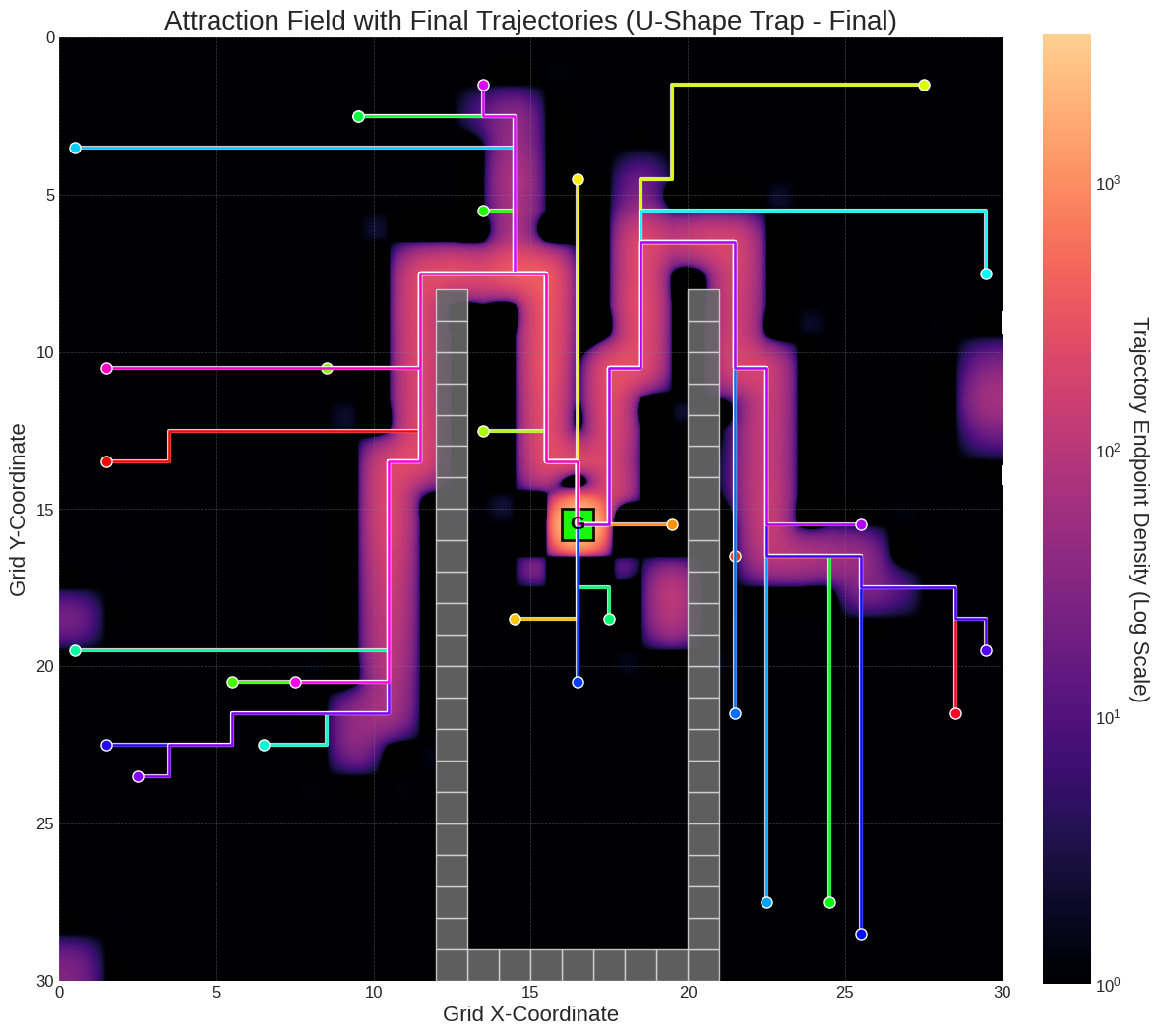}
		\caption{Trajectories and attraction field in U Shaped Trap}
	\end{subfigure}%
	\hfill
	\begin{subfigure}{0.48\textwidth}
		\includegraphics[width=\linewidth, height=0.27\textheight, keepaspectratio]{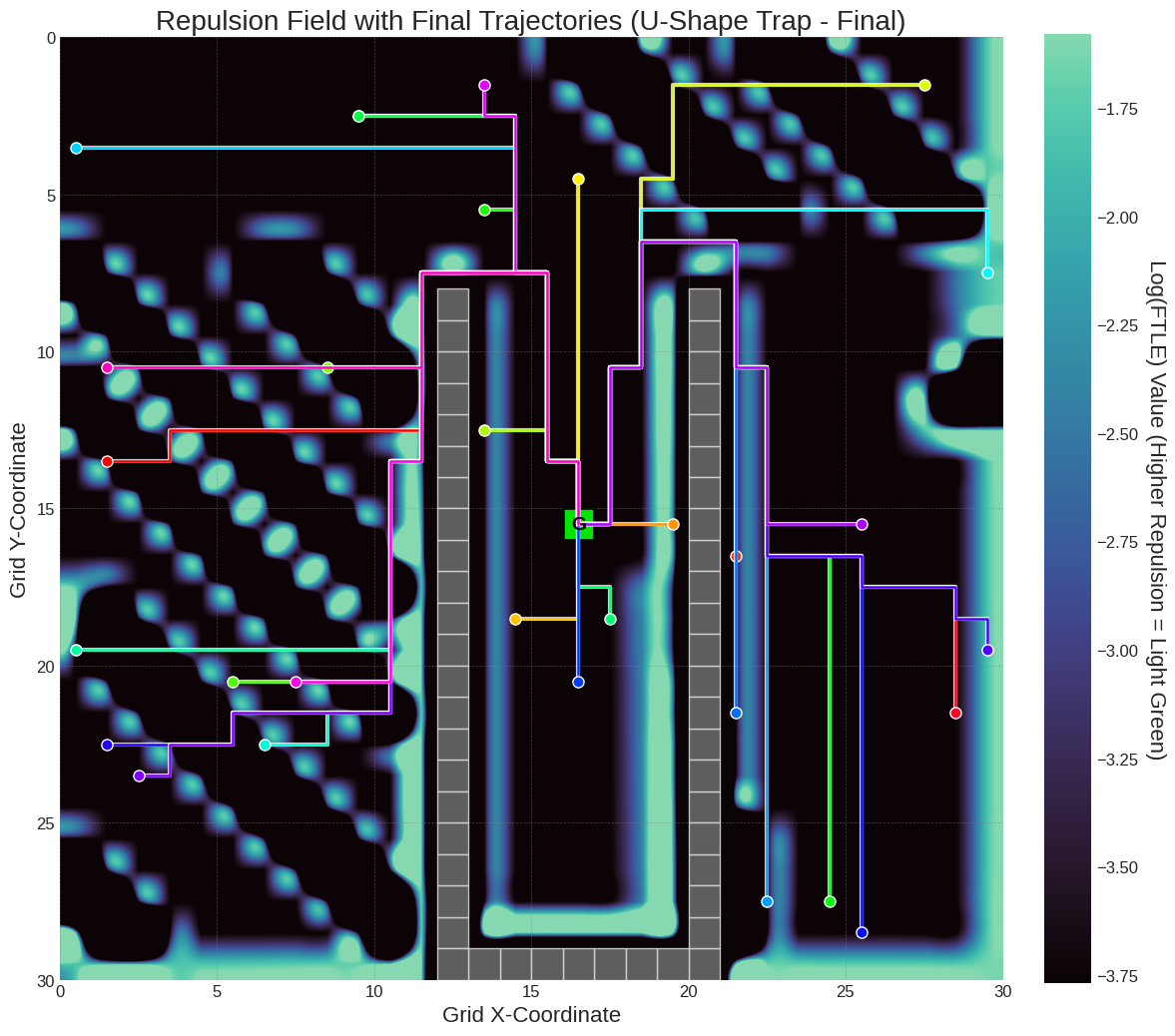}
		\caption{Trajectories and repulsion fields in U Shaped Trap}
	\end{subfigure}
	
	\caption{Trajectories in all three grid environments following the attraction fields and avoiding the repulsion field regions }
	\label{fig:grid}
\end{figure*}

\section{Application to Continuous Control Environments}

To demonstrate the versatility of our framework beyond discrete grid-worlds, we apply it to several classic continuous control environments from the Gymnasium library. For this analysis, we utilize high-performance, pre-trained policies available on the Hugging Face Hub, trained with state-of-the-art algorithms such as Truncated Quantile Critics (TQC), Soft Actor-Critic (SAC), and Twin-Delayed Deep Deterministic Policy Gradient (TD3). This approach allows us to directly analyze well-converged policies and showcases the framework's utility for evaluating agents trained by any standard deep reinforcement learning method.

For environments with a state space dimension greater than two, we analyze 2D slices of the dynamical landscape. This is achieved by fixing the remaining state variables to logical constant values (e.g., zero velocity and zero angle), allowing us to visualize the policy's behavior in a critical subspace of operation.

\subsection{Case Study 1: MountainCarContinuous-v0}

The \texttt{MountainCarContinuous-v0} environment is a classic problem requiring an agent to build momentum to escape a valley. A successful policy must learn a non-trivial, dynamic strategy.

\begin{figure}[!h]
    \centering
    \includegraphics[width=\columnwidth]{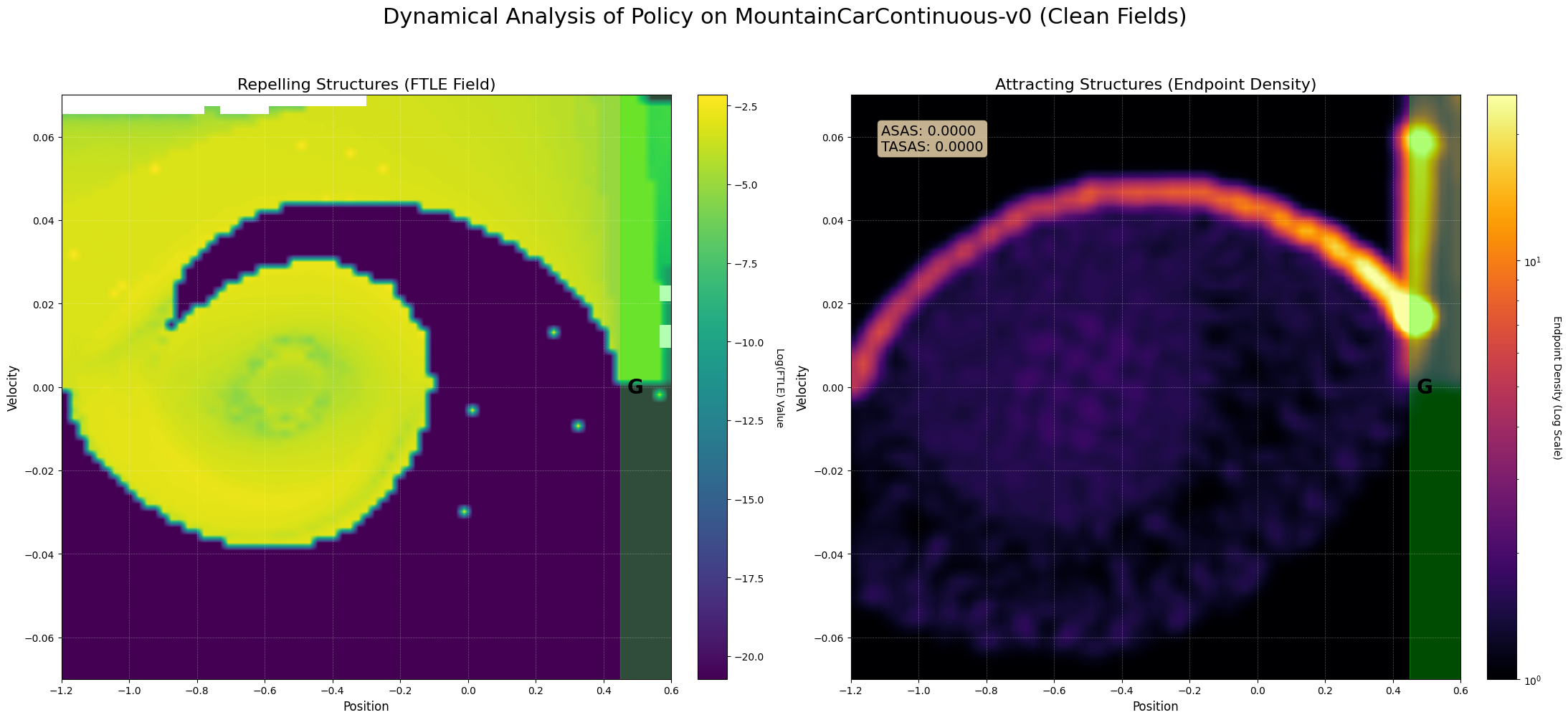}
    \caption{Dynamical analysis of a TQC policy on \texttt{MountainCarContinuous-v0}. The FTLE field (left) reveals a strong repelling boundary corresponding to the valley walls and a spiral structure indicating the momentum-building strategy. The attractor plot (right) shows a clear "highway" of trajectories that successfully reach the goal region (G) on the right.}
    \label{fig:mountain_car_fields}
\end{figure}

\begin{figure}[!h]
    \centering
    \includegraphics[width=\columnwidth]{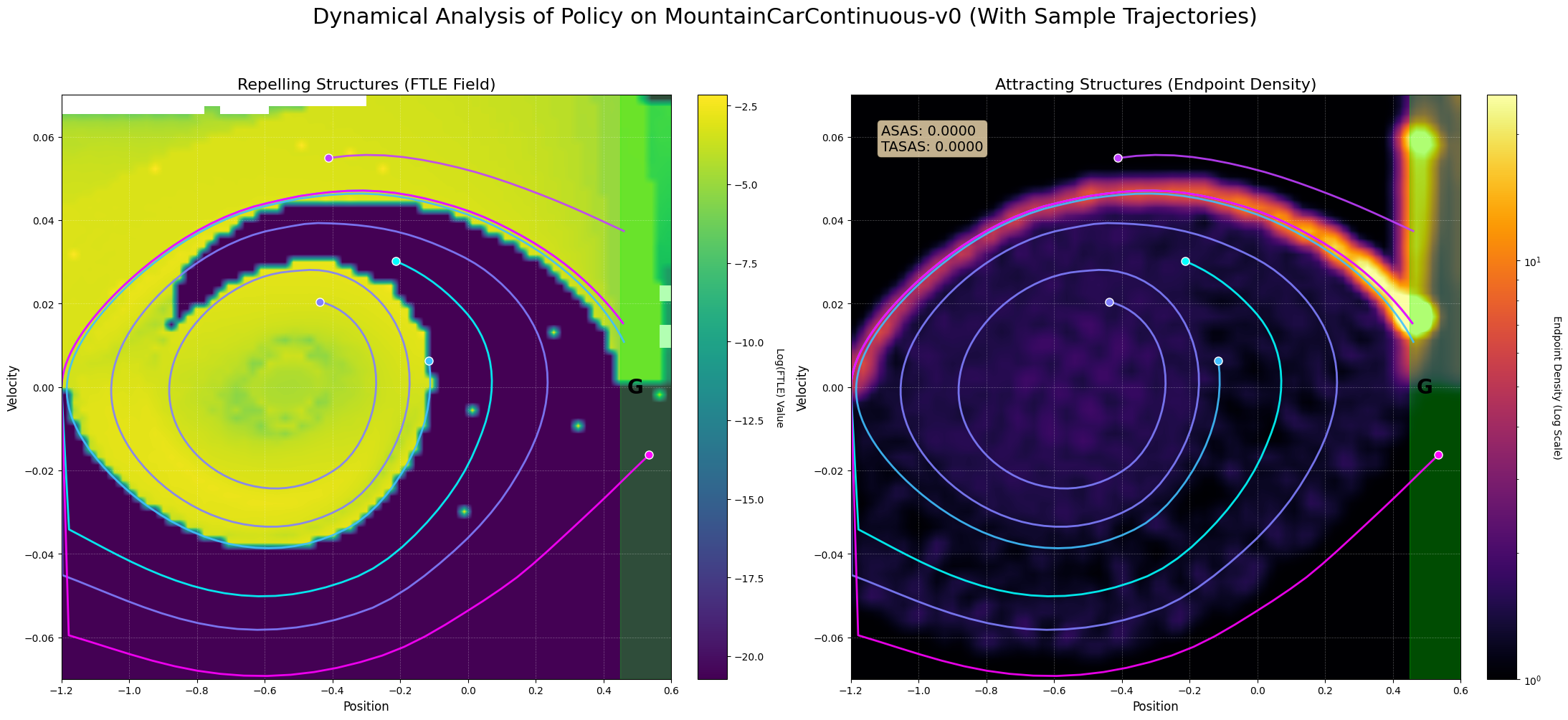}
    \caption{Trajectories on repulsion fields showing how they loop around the spiral avoiding certain regions near 0.00 velocity and gaining momentum and moving towards the goal. On atrraction fields we see transcribers following a specified path with increasing attraction field and moving towards the goal region.}
    \label{fig:mountain_car_traj}
\end{figure}

The analysis in Figure~\ref{fig:mountain_car_fields} reveals a highly robust and successful policy. The FTLE plot on the left clearly delineates the valley walls as strong repelling structures (bright yellow). The spiral pattern in the center of the state space visualizes the agent's learned strategy of oscillating to build velocity. The attractor plot on the right is even more telling: it shows a single, dominant "highway" that collects trajectories from across the state space and guides them directly to the goal region. There are no other significant basins of attraction, indicating that the agent does not get stuck.

\textbf{Quantitative Analysis:} The calculated metrics confirm this visual interpretation. With an ASAS score of 0.0000 and a TASAS score of 0.0000, the policy exhibits ideal robustness. According to our definitions, this signifies the complete absence of significant or persistent spurious attractors. The agent has learned a globally effective strategy.

\subsection{Case Study 2: Pendulum-v1}

The \texttt{Pendulum-v1} environment involves a different challenge: stabilization. The goal is to swing a pendulum upright and maintain its balance at the top (angle = 0, velocity = 0).

\begin{figure}[!h]
    \centering
    \includegraphics[width=\columnwidth]{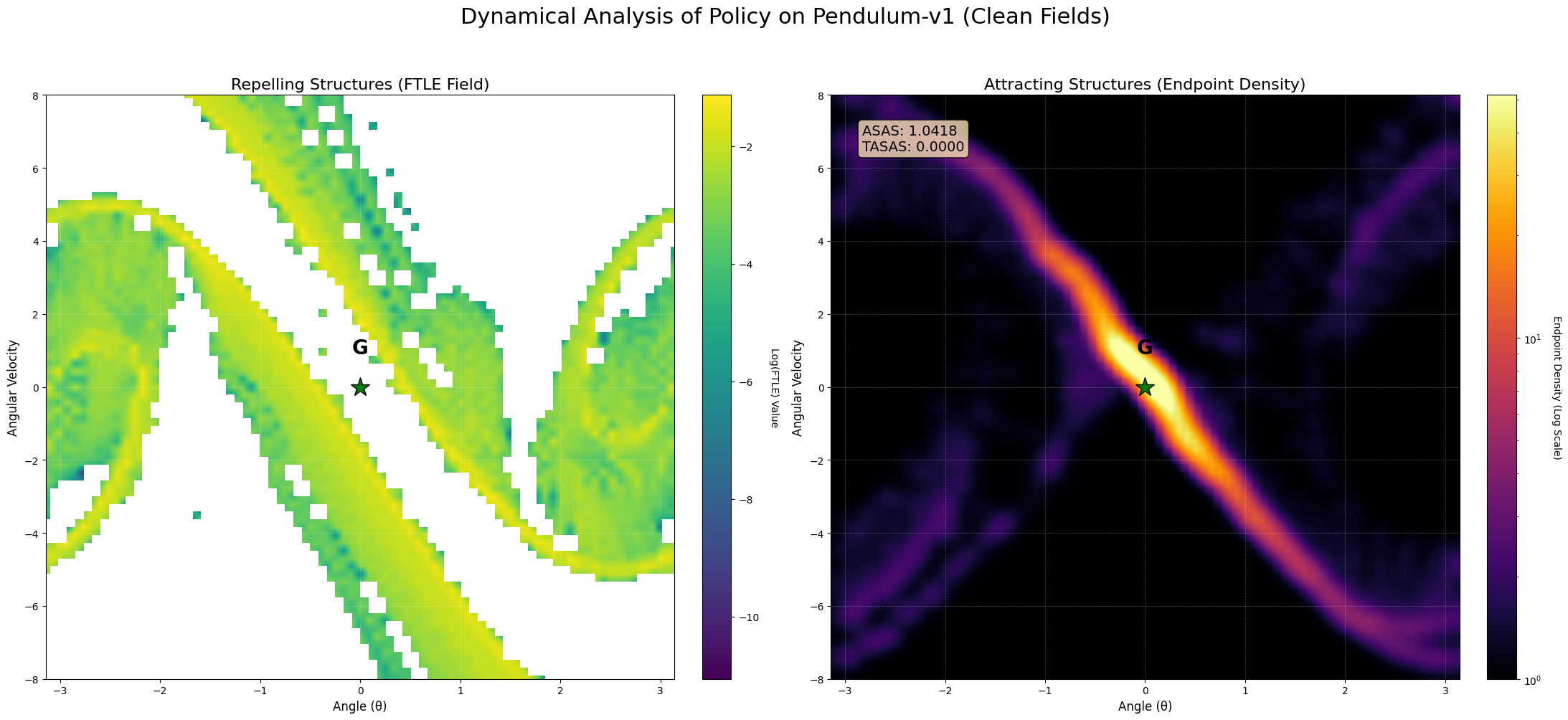}
    \caption{Dynamical analysis of an SAC policy on \texttt{Pendulum-v1}. The FTLE field (left) is largely featureless, indicating stable, non-chaotic behavior. The attractor plot (right) shows an extremely powerful and compact attractor at the goal state (G), representing the upright and stationary position and a "highway" leading to the goal state.}
    \label{fig:pendulum_fields}
\end{figure}

\begin{figure}[!h]
    \centering
    \includegraphics[width=\columnwidth]{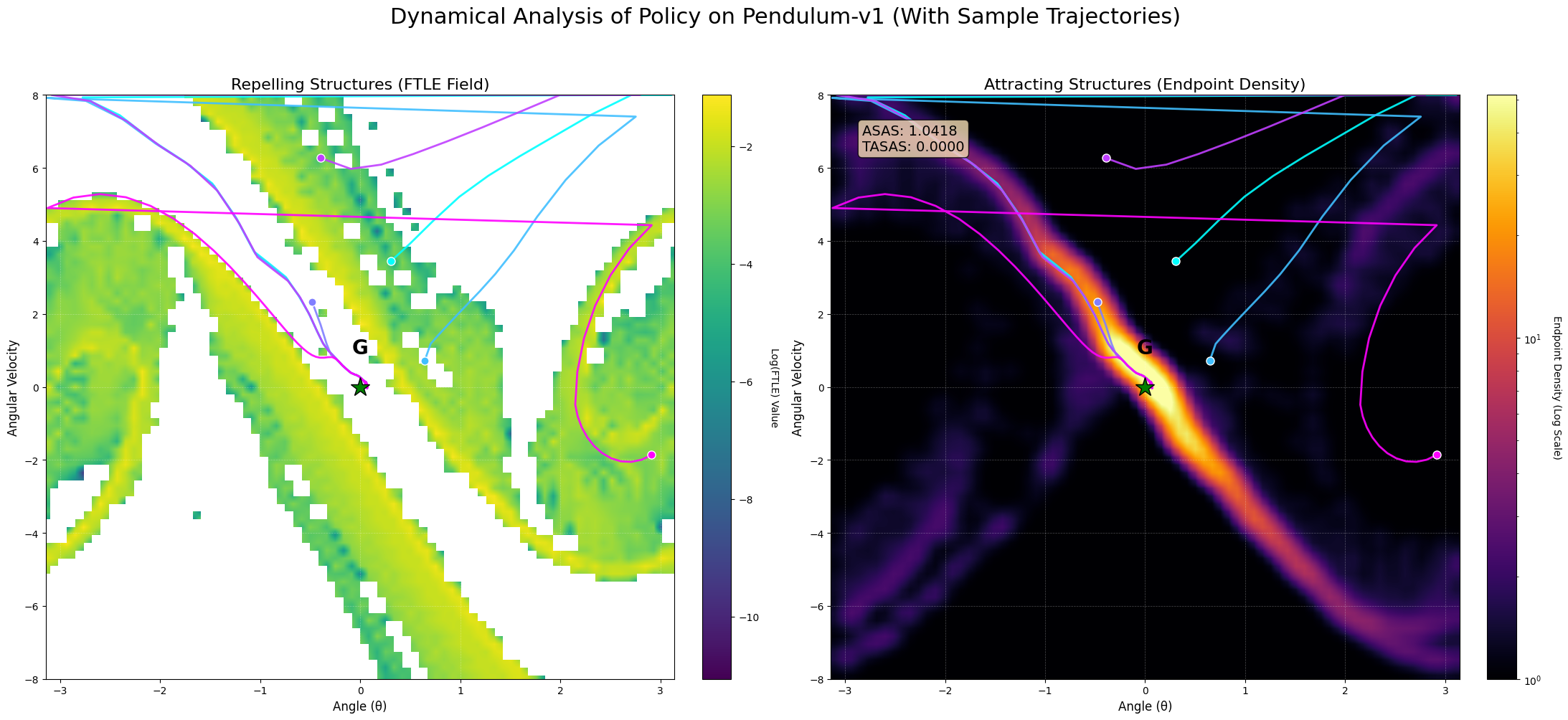}
    \caption{The trajectories follow attraction fields into the goal region while repulsion fields on both sides of the attracting structures make sure the trajectories don't veer to far from the goal region }
    \label{fig:pendulum_traj}
\end{figure}

Figure~\ref{fig:pendulum_fields} illustrates a perfectly learned stabilization policy. The FTLE field is low and flat, which is expected for a stable system where nearby trajectories are not driven apart. The attractor plot is the key result: it shows a single, tight, and dominant hotspot precisely at the goal state. All trajectories in the analyzed state space converge to this single desired equilibrium point.

\textbf{Quantitative Analysis:} The metrics provide definitive confirmation of the policy's quality. Both the ASAS and TASAS scores are 0.0000, signifying perfect robustness. The system contains no competing attractors, and its dynamics are entirely focused on achieving and maintaining the goal state.

\subsection*{Case Study 2.1: Formal Stability Verification for the Pendulum}
We now apply our local stability guarantee (Proposition 4.1) to this successful policy. To demonstrate the utility of our local stability proposition, we analyze the trained Pendulum agent. We define a stable region R as a circle of radius 1.60 around the goal state $([\theta, \omega] = [0, 0])$. From the computed FTLE field, we find the maximum FTLE value within this region to be $\sigma_{\max} \approx 0.7158$ with an integration time of $T_{\text{int}} = 20$ steps.

Using our derived bound, we can formally certify that for any desired final separation tolerance, $\epsilon = 0.05$, the policy is robust to any initial state perturbation smaller than $\delta < 3.0 \times 10^{-8}$. Although the maximum tolerable perturbation is small (due to the exponential nature of the bound and a non-trivial $\sigma_{\max}$), this provides a concrete, formal guarantee of the policy's stability near its equilibrium point. The process is visualized in Figure~\ref{fig:pendulum_stability_viz}.

\begin{figure}[!h]
    \centering
    \includegraphics[width=\columnwidth]{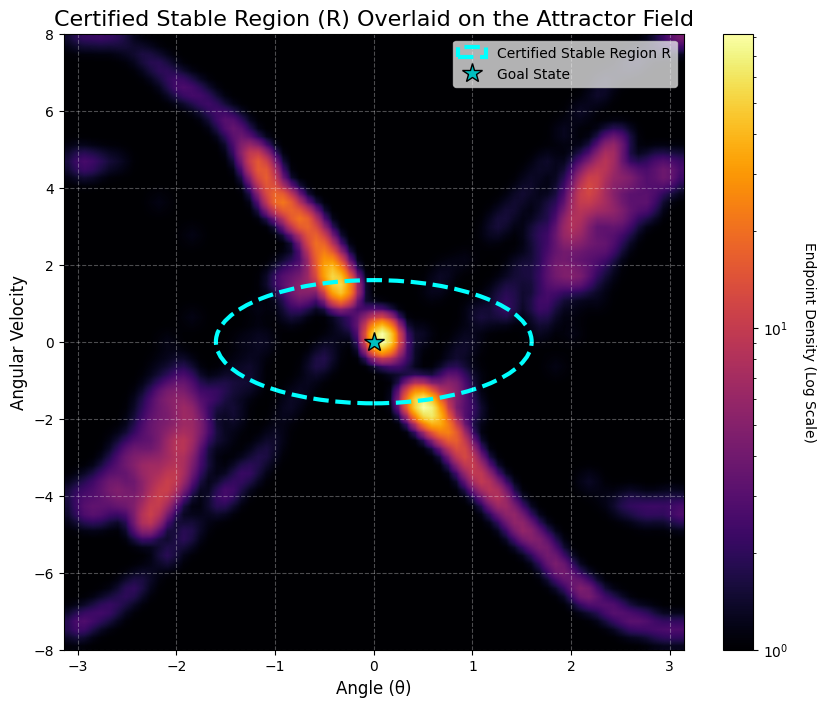}
    \caption{Visualization of the certified stable region R for the Pendulum-v1 policy. The dashed cyan circle represents the region R where the stability guarantee is calculated. This region is overlaid on the learned attractor field, demonstrating that the area of certified stability coincides with the policy's primary goal attractor.}
    \label{fig:pendulum_stability_viz}
\end{figure}

\subsection{Case Study 3: LunarLanderContinuous-v2}
The \texttt{LunarLanderContinuous-v2} environment presents a more complex, 8-dimensional control problem. To analyze it, we examine a 2D slice of the state space corresponding to the lander's (x, y) position, while holding all other states (velocities, angle, leg contact) at zero. This represents the policy's behavior during a critical final descent or hover phase.

\begin{figure}[!h]
    \centering
    \includegraphics[width=\columnwidth]{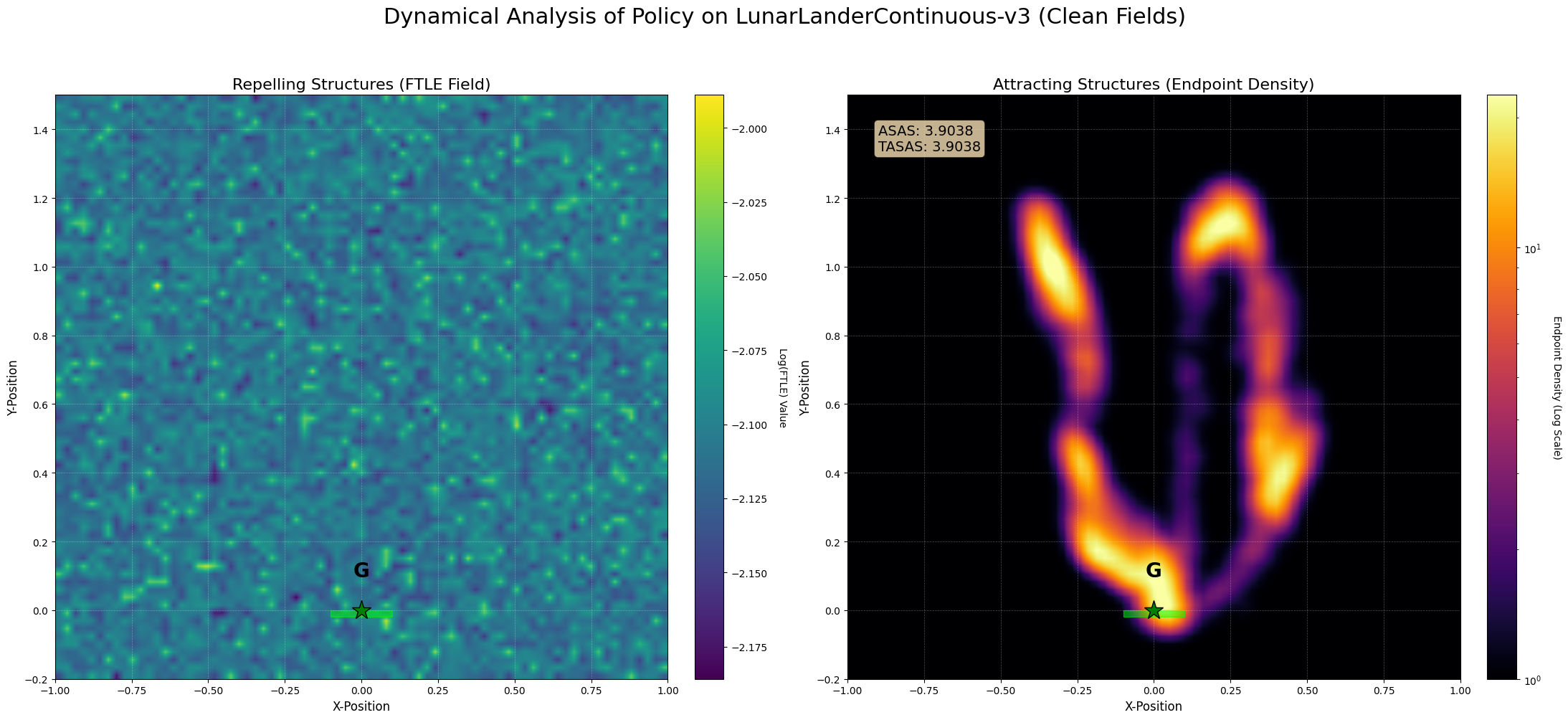}
    \caption{Dynamical analysis of a TD3 policy on a 2D slice of \texttt{LunarLanderContinuous-v2}. The FTLE field (left) is noisy and lacks coherent structures. The attractor plot (right) reveals a critical flaw: while an attractor exists at the goal (G), two powerful spurious attractor "highways" dominate the flow, pulling trajectories away from the landing pad.}
    \label{fig:lunar_lander_fields}
\end{figure}

\begin{figure}[!h]
    \centering
    \includegraphics[width=\columnwidth]{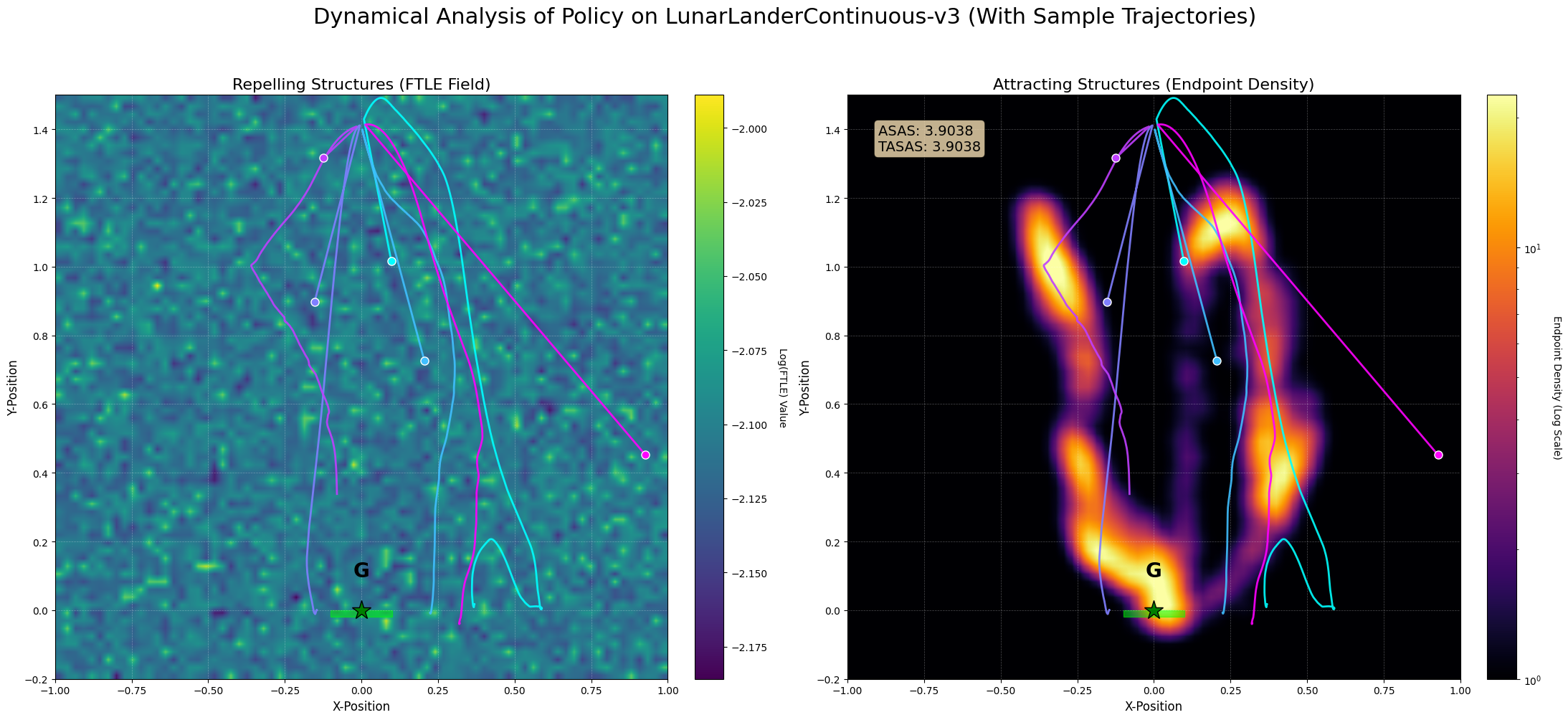}
    \caption{Trajecctoire tend to follow two "highways" to a goal region but there seems to be other regions of high attraction other than the goal. Thus not all trajectories make it to the goal regions and end up getting trapped in the attraction fields.} 
    \label{fig:lunar_lander_traj}
\end{figure}

The analysis in Figure~\ref{fig:lunar_lander_fields} uncovers a critical flaw in the pre-trained policy, demonstrating the diagnostic power of our framework. The FTLE field is noisy and lacks the strong, coherent barriers seen in robust policies. The attractor plot is damning: although some trajectories land at the goal (G), the flow is dominated by two large, bright attractor arms. These structures act as highways that lead trajectories into undesirable looping patterns rather than toward the landing pad.

\textbf{Quantitative Analysis:} The metrics powerfully quantify this failure. The policy yields an ASAS score of 3.9038 and a TASAS score of 3.9038. According to our definitions, an ASAS score greater than 1.0 indicates "extremely poor robustness," meaning the cumulative pull of the spurious attractors is stronger than the pull of the goal. The high TASAS score confirms that these attractors are not transient but are persistent traps from which the agent fails to escape. This policy is demonstrably unsafe, a conclusion that would be difficult to reach without this dynamical systems analysis.

\subsection{Summary of Continuous Control Results}

Table~\ref{tab:continuous_results} summarizes the quantitative safety and robustness metrics for the analyzed policies. The results clearly distinguish between the ideally robust policies for \texttt{MountainCar} and \texttt{Pendulum}, and the critically flawed policy for \texttt{LunarLander}.

\begin{table}[H]
\centering
\caption{Quantitative robustness metrics for continuous control policies.}
\label{tab:continuous_results}
\begin{tabular}{lrr}
\hline
\textbf{Environment} & \textbf{ASAS Score} & \textbf{TASAS Score} \\ \hline
\texttt{MountainCarContinuous-v0} & 0.0000 & 0.0000 \\
\texttt{Pendulum-v1} & 0.0000 & 0.0000 \\
\texttt{LunarLanderContinuous-v2} & 3.9038 & 3.9038 \\ \hline
\end{tabular}
\end{table}

This extension to continuous domains validates that our framework is a general and powerful tool for RL verification, capable of providing deep insights and uncovering critical failure modes that are not apparent from standard performance metrics like cumulative reward.
	\section{Conclusion}
	This paper introduced a novel framework for the verification of safety and robustness in reinforcement learning agents by treating the combined policy-environment system through the lens of dynamical systems theory. By applying the Finite-Time Lyapunov Exponent (FTLE) and analyzing Lagrangian Coherent Structures (LCS), we demonstrated a powerful method for visualizing the underlying dynamical mechanisms of a learned policy. Repelling LCS were shown to correspond to safety barriers, while the system's attractors revealed its convergence properties and potential failure modes.
	
	To formalize this analysis, we developed a suite of quantitative metrics, MBR, ASAS, and TASAS, that provide scalar measures of a policy's safety margin and robustness. These metrics proved effective across a range of environments, successfully distinguishing between safe, robust policies and those with critical, hidden flaws. The framework's ability to identify spurious attractors and quantify their persistence proved essential for uncovering critical flaws in policies that might otherwise appear successful based on reward-based evaluation alone, highlighting its diagnostic power.
	
	The framework's applicability to both discrete and continuous control tasks, as well as its foundation in established mathematical theory, offers a promising path toward building more trustworthy and reliable autonomous systems. Future work could extend this analysis to stochastic systems, explore its use in guiding policy training directly, and develop more computationally efficient methods for high-dimensional state spaces. Ultimately, by bridging the gap between the empirical success of RL and the rigorous demands of formal verification, this dynamical systems perspective provides a crucial tool for deploying intelligent agents safely and confidently in the real world.

    
\bibliographystyle{IEEEtran}
\bibliography{bibfile}

\begin{thebibliography}{10}
\providecommand{\url}[1]{#1}
\csname url@samestyle\endcsname
\providecommand{\newblock}{\relax}
\providecommand{\bibinfo}[2]{#2}
\providecommand{\BIBentrySTDinterwordspacing}{\spaceskip=0pt\relax}
\providecommand{\BIBentryALTinterwordstretchfactor}{4}
\providecommand{\BIBentryALTinterwordspacing}{\spaceskip=\fontdimen2\font plus
\BIBentryALTinterwordstretchfactor\fontdimen3\font minus \fontdimen4\font\relax}
\providecommand{\BIBforeignlanguage}[2]{{%
\expandafter\ifx\csname l@#1\endcsname\relax
\typeout{** WARNING: IEEEtran.bst: No hyphenation pattern has been}%
\typeout{** loaded for the language `#1'. Using the pattern for}%
\typeout{** the default language instead.}%
\else
\language=\csname l@#1\endcsname
\fi
#2}}
\providecommand{\BIBdecl}{\relax}
\BIBdecl

\bibitem{norvig1994artificial}
P.~Norvig and S.~Russell, ``Artificial intelligence: a modern approach,'' \emph{Prentice Hall}, 1994.

\bibitem{shinde2018review}
P.~P. Shinde and S.~Shah, ``A review of machine learning and deep learning applications,'' pp. 1--6, 2018.

\bibitem{lecun2015deep}
Y.~LeCun, Y.~Bengio, and G.~Hinton, ``Deep learning,'' \emph{Nature}, vol. 521, no. 7553, pp. 436--444, 2015.

\bibitem{alzubaidi2021review}
L.~Alzubaidi, J.~Zhang, A.~J. Humaidi, A.~Al-Dujaili, Y.~Duan, O.~Al-Shamma, J.~Santamar{\'\i}a, M.~A. Fadhel, M.~Al-Amidie, and L.~Farhan, ``Review of deep learning: concepts, cnn architectures, challenges, applications, future directions,'' \emph{Journal of big Data}, vol.~8, no.~1, pp. 1--74, 2021.

\bibitem{sutton1998reinforcement}
R.~S. Sutton and A.~G. Barto, \emph{Reinforcement learning: An introduction}.\hskip 1em plus 0.5em minus 0.4em\relax MIT press Cambridge, 1998, vol.~1.

\bibitem{kaelbling1996reinforcement}
L.~P. Kaelbling, M.~L. Littman, and A.~W. Moore, ``Reinforcement learning: A survey,'' \emph{Journal of artificial intelligence research}, vol.~4, pp. 237--285, 1996.

\bibitem{arulkumaran2017deep}
K.~Arulkumaran, M.~P. Deisenroth, M.~Brundage, and A.~A. Bharath, ``Deep reinforcement learning: A brief survey,'' \emph{IEEE Signal Processing Magazine}, vol.~34, no.~6, pp. 26--38, 2017.

\bibitem{prudencio2023survey}
R.~F. Prudencio, M.~R. O.~A. Maximo, and E.~L. Colombini, ``A survey on offline reinforcement learning: Taxonomy, review, and open problems,'' \emph{IEEE Transactions on Neural Networks and Learning Systems}, 2023.

\bibitem{silver2016mastering}
D.~Silver, A.~Huang, C.~J. Maddison, A.~Guez, L.~Sifre, G.~Van Den~Driessche, J.~Schrittwieser, I.~Antonoglou, V.~Panneershelvam, M.~Lanctot \emph{et~al.}, ``Mastering the game of go with deep neural networks and tree search,'' \emph{Nature}, vol. 529, no. 7587, pp. 484--489, 2016.

\bibitem{silver2018general}
D.~Silver, T.~Hubert, J.~Schrittwieser, I.~Antonoglou, M.~Lai, A.~Guez, M.~Lanctot, L.~Sifre, D.~Kumaran, T.~Graepel \emph{et~al.}, ``A general reinforcement learning algorithm that masters chess, shogi, and go through self-play,'' \emph{Science}, vol. 362, no. 6419, pp. 1140--1144, 2018.

\bibitem{vinyals2019grandmaster}
O.~Vinyals, I.~Babuschkin, W.~M. Czarnecki, M.~Mathieu, A.~Dudzik, J.~Chung, D.~H. Choi, R.~Powell, T.~Ewalds, P.~Georgiev \emph{et~al.}, ``Grandmaster level in starcraft ii using multi-agent reinforcement learning,'' \emph{Nature}, vol. 575, no. 7782, pp. 350--354, 2019.

\bibitem{kober2013reinforcement}
J.~Kober, J.~A. Bagnell, and J.~Peters, ``Reinforcement learning in robotics: A survey,'' \emph{The International Journal of Robotics Research}, vol.~32, no.~11, pp. 1238--1274, 2013.

\bibitem{openai2019solving}
{OpenAI}, I.~Akkaya, M.~Andrychowicz, M.~Chociej, M.~Litwin, B.~McGrew, A.~Petron, A.~Paino, M.~Plappert, G.~Powell, R.~Ribas, J.~Schneider, N.~Tezak, J.~Tworek, P.~Welinder, L.~Weng, Q.~Yuan, W.~Zaremba, and L.~Zhang, ``Solving rubik's cube with a robot hand,'' \emph{arXiv preprint arXiv:1910.07113}, 2019.

\bibitem{kiran2021deep}
B.~R. Kiran, I.~Sobh, V.~Talpaert, P.~Mannion, A.~A. Sallab, S.~Yogamani, and P.~P{\'e}rez, ``Deep reinforcement learning for autonomous driving: A survey,'' \emph{IEEE Transactions on Intelligent Transportation Systems}, vol.~23, no.~6, pp. 4909--4926, 2021.

\bibitem{shalev2016safe}
S.~Shalev-Shwartz, S.~Shammah, and A.~Shashua, ``Safe, multi-agent, reinforcement learning for autonomous driving,'' \emph{arXiv preprint arXiv:1610.03295}, 2016.

\bibitem{krishna2023finite}
K.~Krishna, S.~L. Brunton, and Z.~Song, ``Finite time lyapunov exponent analysis of model predictive control and reinforcement learning,'' \emph{arXiv preprint arXiv:2304.03326}, 2023.

\bibitem{zolman2024sindyrl}
N.~Zolman, U.~Fasel, J.~N. Kutz, and S.~L. Brunton, ``Sindy-rl: Interpretable and efficient model-based reinforcement learning,'' \emph{arXiv preprint arXiv:2403.09110}, 2024.

\bibitem{mnih2013playing}
V.~Mnih, K.~Kavukcuoglu, D.~Silver, A.~Graves, I.~Antonoglou, D.~Wierstra, and M.~Riedmiller, ``Playing atari with deep reinforcement learning,'' \emph{arXiv preprint arXiv:1312.5602}, 2013.

\bibitem{schulman2017proximal}
J.~Schulman, F.~Wolski, P.~Dhariwal, A.~Radford, and O.~Klimov, ``Proximal policy optimization algorithms,'' \emph{arXiv preprint arXiv:1707.06347}, 2017.

\bibitem{haarnoja2018soft}
T.~Haarnoja, A.~Zhou, P.~Abbeel, and S.~Levine, ``Soft actor-critic: Off-policy maximum entropy deep reinforcement learning with a stochastic actor,'' in \emph{International conference on machine learning}.\hskip 1em plus 0.5em minus 0.4em\relax PMLR, 2018, pp. 1861--1870.

\bibitem{doshi2017towards}
F.~Doshi-Velez and B.~Kim, ``Towards a rigorous science of interpretable machine learning,'' \emph{arXiv preprint arXiv:1702.08608}, 2017.

\bibitem{pullum2020review}
L.~L. Pullum, ``Review of metrics to measure the stability, robustness and resilience of reinforcement learning,'' \emph{Preprint, available at Oak Ridge National Laboratory}, 2020.

\bibitem{attar2019reinforcement}
M.~Attar and M.~R. Dabirian, ``Reinforcement learning for learning of dynamical systems in uncertain environment: A tutorial,'' \emph{arXiv preprint arXiv:1907.03154}, 2019.

\bibitem{zhang2020robust}
H.~Zhang, H.~Chen, C.~Xiao, B.~Li, M.~Liu, D.~Boning, and C.-J. Hsieh, ``Robust deep reinforcement learning against adversarial perturbations on state observations,'' in \emph{Advances in Neural Information Processing Systems}, vol.~33, 2020, pp. 21\,024--21\,037.

\bibitem{schott2024robust}
L.~Schott, J.~Delas, H.~Hajri, E.~Gherbi, R.~Yaich, N.~Boulahia-Cuppens, F.~Cuppens, and S.~Lamprier, ``Robust deep reinforcement learning through adversarial attacks and training: A survey,'' \emph{arXiv preprint arXiv:2403.00420}, 2024.

\bibitem{fischer2019online}
M.~Fischer, M.~Mirman, S.~Stalder, and M.~Vechev, ``Online robustness training for deep reinforcement learning,'' \emph{arXiv preprint arXiv:1911.00887}, 2019.

\bibitem{zhang2021robust}
H.~Zhang, H.~Chen, D.~Boning, and C.-J. Hsieh, ``Robust reinforcement learning on state observations with learned optimal adversary,'' \emph{arXiv preprint arXiv:2101.08452}, 2021.

\bibitem{dabholkar2023adversarial}
A.~Dabholkar, J.~Z. Hare, M.~Mittrick, J.~Richardson, N.~Waytowich, P.~Narayanan, and S.~Bagchi, ``Adversarial attacks on reinforcement learning agents for command and control,'' \emph{arXiv preprint arXiv:2405.01693}, 2023.

\bibitem{young2024enhancing}
R.~Young and N.~Pugeault, ``Enhancing robustness in deep reinforcement learning: A lyapunov exponent approach,'' \emph{arXiv preprint arXiv:2410.10674}, 2024.

\bibitem{chow2018lyapunov}
Y.~Chow, O.~Nachum, E.~Duenez-Guzman, and M.~Ghavamzadeh, ``A lyapunov-based approach to safe reinforcement learning,'' in \emph{Advances in neural information processing systems}, vol.~31, 2018.

\bibitem{achiam2017constrained}
J.~Achiam, D.~Held, A.~Tamar, and P.~Abbeel, ``Constrained policy optimization,'' in \emph{International conference on machine learning}.\hskip 1em plus 0.5em minus 0.4em\relax PMLR, 2017, pp. 22--31.

\bibitem{pinto2017robust}
L.~Pinto, J.~Davidson, R.~Sukthankar, and A.~Gupta, ``Robust adversarial reinforcement learning,'' in \emph{International conference on machine learning}.\hskip 1em plus 0.5em minus 0.4em\relax PMLR, 2017, pp. 2817--2826.

\bibitem{li2023safe}
Z.~Li, C.~Hu, Y.~Wang, Y.~Yang, and S.~E. Li, ``Safe reinforcement learning with dual robustness,'' \emph{arXiv preprint arXiv:2309.06835}, 2023.

\bibitem{bacci2022verified}
E.~Bacci and D.~Parker, ``Verified probabilistic policies for deep reinforcement learning,'' in \emph{Computer Aided Verification: 34th International Conference, CAV 2022}.\hskip 1em plus 0.5em minus 0.4em\relax Springer, 2022, pp. 527--547.

\bibitem{wu2024verified}
J.~Wu, H.~Zhang, and Y.~Vorobeychik, ``Verified safe reinforcement learning for neural network dynamic models,'' \emph{arXiv preprint arXiv:2405.15994}, 2024.

\bibitem{haller2015lagrangian}
G.~Haller, ``Lagrangian coherent structures,'' \emph{Annual review of fluid mechanics}, vol.~47, pp. 137--162, 2015.

\bibitem{shadden2005definition}
S.~C. Shadden, F.~Lekien, and J.~E. Marsden, ``Definition and properties of lagrangian coherent structures from finite-time lyapunov exponents in two-dimensional aperiodic flows,'' \emph{Physica D: Nonlinear Phenomena}, vol. 212, no. 3-4, pp. 271--304, 2005.

\bibitem{jones2024mode}
M.~R. Jones, C.~Klewicki, O.~Khan, S.~L. Brunton, and M.~Luhar, ``Mode sensitivity: Connecting lagrangian coherent structures with modal analysis for fluid flows,'' \emph{arXiv preprint arXiv:2410.20802}, 2024.

\end{thebibliography}
	
\end{document}